\documentclass{article}

\usepackage[preprint]{neurips_2022}

\usepackage[utf8]{inputenc} % allow utf-8 input
\usepackage[T1]{fontenc}    % use 8-bit T1 fonts
\usepackage{hyperref}       % hyperlinks
\usepackage{url}            % simple URL typesetting
\usepackage{booktabs}       % professional-quality tables
\usepackage{amsfonts}       % blackboard math symbols
\usepackage{nicefrac}       % compact symbols for 1/2, etc.
\usepackage{microtype}      % microtypography
\usepackage{xcolor}         % colors
\usepackage{times}
\usepackage[linesnumbered, ruled,vlined]{algorithm2e}
\usepackage{amsthm}
\usepackage{wrapfig}

\newtheorem{lemma}{Lemma}[section]
\newtheorem{theorem}[lemma]{Theorem}

\newtheorem{definition}[lemma]{Definition}

\usepackage{times}
\usepackage{enumitem}
\usepackage{mathtools}

\usepackage{thm-restate}

\usepackage{url}
\usepackage{dsfont}
\usepackage[mathscr]{euscript}
\usepackage{booktabs}
\usepackage{makecell}
\usepackage{xcolor, colortbl}
\usepackage[normalem]{ulem}
\usepackage{soul}
\usepackage{prettyref}

\definecolor{DarkRed}{rgb}{0.75,0,0}
\definecolor{DarkGreen}{rgb}{0,0.5,0}
\definecolor{DarkPurple}{rgb}{0.5,0,0.5}
\definecolor{DarkBlue}{rgb}{0,0,0.7}

\hypersetup{colorlinks=true, plainpages = false, citecolor = DarkBlue, urlcolor = DarkBlue, filecolor = black, linkcolor =DarkBlue}

\newcommand{\Acal}{\mathcal{A}}

\newcommand{\Dcal}{\mathcal{D}}

\newcommand{\Ecal}{\mathcal{E}}

\newcommand{\Bcal}{\mathcal{B}}

\newcommand{\Ical}{\mathcal{I}}

\newcommand{\Ncal}{\mathcal{N}}

\newcommand{\EE}{\mathbb{E}} % Expectation
\newcommand{\PP}{\mathbb{P}} % Probability
\newcommand{\RR}{\mathbb{R}} % Real numbers
\newcommand{\sA}{{\mathscr A}}

\newcommand{\Ocal}{\mathcal{O}}

\DeclareMathOperator*{\E}{\mathbb E}
\DeclareMathOperator*{\argmin}{argmin}

\DeclareMathOperator{\Reg}{\mathsf{Reg}}

\DeclareMathOperator{\gap}{gap}

\DeclareMathOperator{\Ind}{\mathbb{I}}

\DeclarePairedDelimiter{\bracket}{[}{]}
\DeclarePairedDelimiter{\curl}{\{}{\}}
\DeclarePairedDelimiter{\paren}{(}{)}

\DeclarePairedDelimiter{\tri}{\langle}{\rangle}

\newcommand{\ignore}[1]{}

\newcommand{\e}{\epsilon}

\usepackage{MnSymbol}
\DeclareMathAlphabet\mathbb{U}{msb}{m}{n}
\usepackage{xpatch}

\usepackage[toc, page, header]{appendix}
\title{Stochastic Online Learning with Feedback Graphs:\\
      Finite-Time and Asymptotic Optimality}

\author{%
  Teodor V. Marinov\\
  Google Research\\
  \texttt{tvmarinov@google.com}\\
   \And
   Mehryar Mohri\\
   Courant Institute and Google Research\\
   \texttt{mohri@google.com} \\
   \AND
   Julian Zimmert\\
   Google Research\\
   \texttt{zimmert@google.com}
}

\begin{document}

\maketitle
\begin{abstract}
  We revisit the problem of stochastic online learning with feedback
  graphs, with the goal of devising algorithms that are optimal, up to
  constants, both asymptotically and in finite time.  We show that,
  surprisingly, the notion of optimal finite-time regret is not a
  uniquely defined property in this context and that, in general, it
  is decoupled from the asymptotic rate.  We discuss alternative
  choices and propose a notion of finite-time optimality that we argue
  is \emph{meaningful}. For that notion, we give an algorithm that
  admits quasi-optimal regret both in finite-time and asymptotically.
\end{abstract}

\section{Introduction}

Online learning is a sequential decision making game in which, at each
round, the learner selects one arm (or expert) out of a finite set of
$K$ arms. In the stochastic setting, each arm admits some reward
distribution and the learner receives a reward drawn from the
distribution corresponding to the arm selected. In the \emph{bandit
setting}, the learner observes only that reward
\citep{lai1985asymptotically,auer2002finite,auer2002nonstochastic}, while in the \emph{full
information setting}, the rewards of all $K$ arms are observed
\citep{littlestone1994weighted,freund1997decision}.

Both settings are special instances of a more general model of online
learning with side information introduced by
\citet{mannor2011bandits}, where the information supplied to the
learner is specified by a \emph{feedback graph}. In an undirected
feedback graph, each vertex represents an arm and an edge between
between arm $v$ and $w$ indicates that the reward of $w$ is observed
when $v$ is selected and vice-versa. The bandit setting corresponds to
a graph reduced to self-loops at each vertex, the full information to
a fully connected graph.
The problem of online learning with stochastic rewards and feedback
graphs has been studied by several publications in the last decade or
so. The performance of an algorithm in this problem is expressed in
terms of its pseudo-regret, that is the different between the expected
reward achieved by always pulling the best arm and the expected
cumulative reward obtained by the algorithm.

The \textsc{ucb} algorithm of \citet{auer2002finite} designed for the
bandit setting forms a baseline for this scenario. For general
feedback graphs, \citet{caron2012leveraging} designed a
\textsc{ucb}-type algorithm, \textsc{ucb-n}, as well as a closely
related variant. The pseudo-regret guarantee of \textsc{ucb-n} is
expressed in terms of the most favorable \emph{clique covering} of the
graph, that is its partitioning into cliques. This guarantee is always
at least as favorable as the bandit one \citep{auer2002finite}, which
coincides with the specific choice of the trivial clique
covering. However, the bound depends on the ratio of the maximum and
minimum mean reward gaps within each clique, which, in general, can be
quite large.

\citet{cohen2016online} presented an action-elimination-type algorithm
\citep{EvenDarMannorMansour2006}, 
whose guarantee depends on the least favorable maximal independent
set.  
While there are instances in which this guarantee is worse compared to the bound presented in \citet{caron2012leveraging}, in general it could be much more favorable compared to the clique partition guarantee of \cite{caron2012leveraging}.
The algorithm of \citet{cohen2016online} does
not require access to the full feedback graph, but only to the
out-neighborhood of the arm selected at each round and the results
also hold for time-varying graphs.
Later, \citet{lykouris2020feedback} presented an improved analysis of
the \textsc{ucb-n} algorithm based on a new layering technique, which
showed that \textsc{ucb-n} benefits, in fact, from a more favorable
guarantee based on the independence number of the graph, at the price
of some logarithmic factors. Their analysis also implied a similar
guarantee for a variant of arm-elimination and Thompson sampling, as
well as some improvement of the bound of \citet{cohen2016online} in
the case of a fixed feedback graph.
\citet{buccapatnam2014stochastic} gave an action-elimination-type
algorithm \citep{EvenDarMannorMansour2006}, \textsc{ucb-lp}, that
leverages the solution of a linear-programming (LP) problem. The
guarantee presented depends only on the domination number of the
graph, which can be substantially smaller than the independence
number. A follow-up publication
\citep{BuccapatnamLiuEryilmazShroff2017} presents an analysis for an
extension of the scenario of online learning with stochastic feedback
graphs.

We will show that the algorithms just discussed do not achieve
asymptotically optimal pseudo-regret guarantees and that it is also
unclear how tight their finite-time instance-dependent bounds are.
\citet{wu2015online} and \citet{li2020stochastic} proposed
asymptotically optimal algorithms with matching lower bounds. However,
the corresponding finite-time regret guarantees are far from optimal
and include terms that can dominate the pseudo-regret for any
reasonable time horizon.

We briefly discuss other work related to online learning with feedback graphs. When rewards are adversarial, there has been a vast amount of work studying different settings for the feedback graph such as the graph evolving throughout the game or the graph not being observable before the start of each round \citep{alon2013bandits,alon2015online,alon2017nonstochastic}. The setting in which only noisy feedback is provided by the graph is addressed in \citet{kocak2016online}. First order regret bounds, that is bounds which depend on the reward of the best arm, are derived in \citet{lykouris2018small,lee2020closer}. 
The setting of sleeping experts is studied in \citet{cortes2019online}.
\citet{cortes2020online} study stochastic rewards when the feedback graph evolves throughout the game, however, they do not assume that the rewards and the graph are statistically independent. Another instance in which the feedback and rewards are correlated is that of online learning with abstention~\citep{cortes2018abstention}. In this setting the player can choose to abstain from making a prediction.
The more general problem of Reinforcement Learning with graph feedback has been studied by \citet{dann2020reinforcement}. For additional work on online learning with feedback graphs we recommend the survey of \citet{valko2016bandits}. 

We revisit the problem of stochastic online learning with feedback
graphs, with the goal of devising algorithms that are optimal, up to
constants, both asymptotically and in finite time.  We show that,
surprisingly, the notion of optimal finite-time regret is not a
uniquely defined property in this context and that, in general, it is
decoupled from the asymptotic rate.
Let $T$ denote the time horizon and $\Reg_{\sA}(T)$ the pseudo-regret of
algorithm $\sA$ after $T$ rounds. When $\sA$ is clear from the context, we drop the subscript. It is known that $c^*$, the value of
the LP considered by \citet{buccapatnam2014stochastic,
  wu2015online,li2020stochastic}, is asymptotically a lower bound for
$\Reg_\sA(T)/\log (T)$.  We prove that no algorithm $\sA$ can achieve
a finite-time pseudo-regret guarantee of the form $\Reg_\sA(T) \leq
O(c^* \!\log(T))$.  Moreover, we show that there exists a feedback
graph $G$ for which \emph{any} algorithm suffers a regret of at least
$\Omega\paren*{K^{\frac{1}{8}} \paren*{c^* \log(T) +
    \frac{1}{\Delta_{\min}}}}$, where $\Delta_{\min}$ is the minimum
reward gap.
We discuss alternative choices and propose a notion of finite-time
optimality that we argue is \emph{meaningful}, based on a regret
quantity $d^*$ that we show any algorithm must incur in the worst
case.  For that notion, we give an algorithm whose pseudo-regret is
quasi-optimal, both in finite-time and asymptotically and can be upper
bounded by $O(c^* \log(T) + d^*)$.

\vspace*{-5pt}
%%%%%%%%%%%%%%%%%%%%%%%%%%%%%%%%%%%%%%%%%%%%%%%%%%%%%%%%%%%%
\section{Learning scenario}
We consider the problem of online learning with stochastic rewards and
a fixed undirected feedback graph. As in the familiar multi-armed
bandit problem, the learner can choose one of $K \geq 1$ arms. Each
arm $i \in [K]$ admits a reward distribution, with mean $\mu_i$. 
For all our lower bounds, we assume that the distribution of the reward of each arm is Gaussian with variance $1/\sqrt{2}$. For our upper bounds, we only assume that the distribution of each arm is sub-Gaussian with variance proxy bounded by $1$. We assume that the means are always bounded in $[0, 1]$.
For arm $i$, we denote by $\Delta_i = \mu^* - \mu_i$ its mean gap to the best
$\mu^* = \max_{i \in [K]} \mu_i$. We will also denote by
$\Delta_{\min}$ the smallest and by $\Delta_{\max}$ the largest of
these gaps.
At each round $t \in [T]$, the learner selects an arm $i_t$ and
receives a reward $r_{t, i_t}$ drawn from the reward distribution of
arm $i_t$. In addition to observing that reward, the learner observes
the reward of some other arms, as specified by an undirected graph $G
= (V, E)$, where the vertex set $V$ coincides with $[K]$: an edge $e
\in E$ between vertices $i$ and $j$ indicates that the learner
observes the reward of arm $j$ when selecting arm $i$ and
vice-versa. We will denote by $N_i$ the set of neighbors of arm $i$ in
$G$, $N_i = \curl*{j \in V\colon (i, j) \in E}$, and will assume
self-loops at every vertex, that is, we have $i \in N_i$ for all $i
\in V$. The objective of the learner $\sA$ is to minimize its
\emph{pseudo-regret}, that is the expected cumulative gap between the
reward of an optimal arm $i^*$ and its reward:
\[
\Reg(T)
= \E\bracket*{\sum_{t = 1}^T (r_{t, i^*} - r_{t, i_t})}
= \mu^* T - \E\bracket*{\sum_{t = 1}^T r_{t, i_t}},
\]
where the expectation is taken over the random draw of a reward from
an arm's distribution and the possibly randomized selection strategy
of the learner.  
In the following, we may sometimes abusively use the shorter  
term regret instead of pseudo-regret.
We will denote by $I^*$ the set of optimal arms, that
is, arms with mean reward $\mu^*$, and, for any $t \in [T]$ will
denote by $r_t$ the vector of all rewards $r_{t, i}$ at time $t$. When discussing asymptotic or finite-time optimality, we assume the setting of Gaussian rewards.

We will assume an \emph{informed setting} where the graph $G$ is fixed
and accessible to the learner before the start of the game.
Our analysis makes use of the following standard graph theory notions
\citep{GoddardHenning2013}.
A subset of the vertices is \emph{independent} if no two vertices
in it are adjacent. The \emph{independence number} of $G$, $\alpha(G)$,
is the size of the maximum independent set in $G$.
A \emph{dominating set} of $G$ is a subset $S \subseteq V$ such that
every vertex not in $S$ is adjacent to $S$. The \emph{domination
number} of $G$, $\gamma(G)$, is the minimum size of a dominating set.
It is known that for any graph $G$, we have $\gamma(G) \leq
\alpha(G)$. The difference between the domination and independence
numbers can be substantial in many cases. For example, for a star
graph with $n$ vertices, we have $\gamma(G) = 1$ and $\alpha(G) = n -
1$. In the following, in the absence of any ambiguity, we simply
drop the graph arguments and write $\alpha$ or $\gamma$.
We will denote by $\Dcal(G')$ the minimum dominating set of a
sub-graph $G' \subseteq G$ and by $\Ical(G')$ the maximum independent
set. When the minimum dominating set is not unique, $\Dcal(G')$ can be
selected in an arbitrary but fixed way.

%%%%%%%%%%%%%%%%%%%%%%%%%%%%%%%%%%%%%%%%%%%%%%%%%%%%%%%%%%%%
\section{Sub-optimality of previous algorithms}

In this section, we discuss in more detail the previous work the most
closely related to ours \citep{buccapatnam2014stochastic,wu2015online,
  buccapatnam2017reward,li2020stochastic} and demonstrate their
sub-optimality. These algorithms all seek to achieve
instance-dependent optimal regret bounds by solving and playing
according to the following linear program (LP), which is known to
characterize the instance-dependent asymptotic regret for this
problem when the rewards follow a Gaussian distribution:
\begin{equation}
\label{eq:lp}
  c^*(\Delta,G) :=\min_{x \in \RR^K_{+}} \ \tri*{x, \Delta} \qquad
  s.t. \ \sum_{j \in N_i} x_j \geq \frac{1}{\Delta_i^2},
  \ \forall i \in [K] \setminus I^*.\tag{LP1}
\end{equation}
\ignore{
\begin{equation}
\label{eq:lp}
\begin{aligned}
  \min_{x \in \RR^K_{+}} & \ \tri*{x, \Delta}\\
  s.t. & \ \sum_{j \in N_i} x_j \geq \frac{1}{\Delta_i^2},
  \quad \forall i \in [K] \setminus I^*.
\end{aligned}
\end{equation}
}
We note
that these prior work algorithms can work in more general settings,
but we will restrict our discussion to their use in the informed
setting with a fixed feedback graph that we consider in this study.

The \textsc{ucb-lp} algorithm of \cite{buccapatnam2014stochastic,
  buccapatnam2017reward} is based on the following modification of
\ref{eq:lp}: $\min_{x \in \RR^K_{+}} \tri*{x, 1}$ subject to
$\sum_{j \in N_i} x_j \geq 1$, for all $i \in [K]$, in which the gap
information is eliminated, working with gaps such that $\Delta_{\min}
= \Theta(\Delta_{\max})$. This modified problem is the LP relaxation
of the minimum dominating set integer program of graph $G$.

The algorithm first solves this minimum dominating set relaxation and
then proceeds as an action elimination-algorithm in $O(\log(T))$
phases. During the first $O(\log(K))$ rounds, their algorithm plays by
exploring based on the solution of their LP. Once the exploration
rounds have concluded, it simply behaves as a bandit
action-elimination algorithm. We argue below that this algorithm is
sub-optimal, in at least two ways.

\textbf{Star graph with equal gaps.}  Consider the case where the
feedback graph is a \emph{star graph} (Figure~\ref{fig:examples}(a)):
there is one root or revealing vertex $r$ adjacent to all other
vertices.  In our construction, the optimal arm is chosen uniformly at
random among the leaves of the graph. The rewards are chosen so that
all sub-optimal arms admit the same expected reward with gap to the
best $\Delta \leq O(1/K^{1+\epsilon})$, $\e > 0$. In this case, an
optimal strategy consists of playing the revealing arm for
$\Theta(1/\Delta^2)$ rounds to identify the optimal arm, and thus
incurs regret at most $O(\frac{1}{\Delta})$. On the other hand the
\textsc{ucb-lp} strategy incurs regret at least $\Omega(\frac{K
  \log(T)}{\Delta})$. Even if we ignore the dependence on the
time horizon, the dependence on $K$ is clearly sub-optimal.

\begin{figure}[t]
\vskip -.35in
\centering
\begin{tabular}{cc}
        \includegraphics[scale=.15]{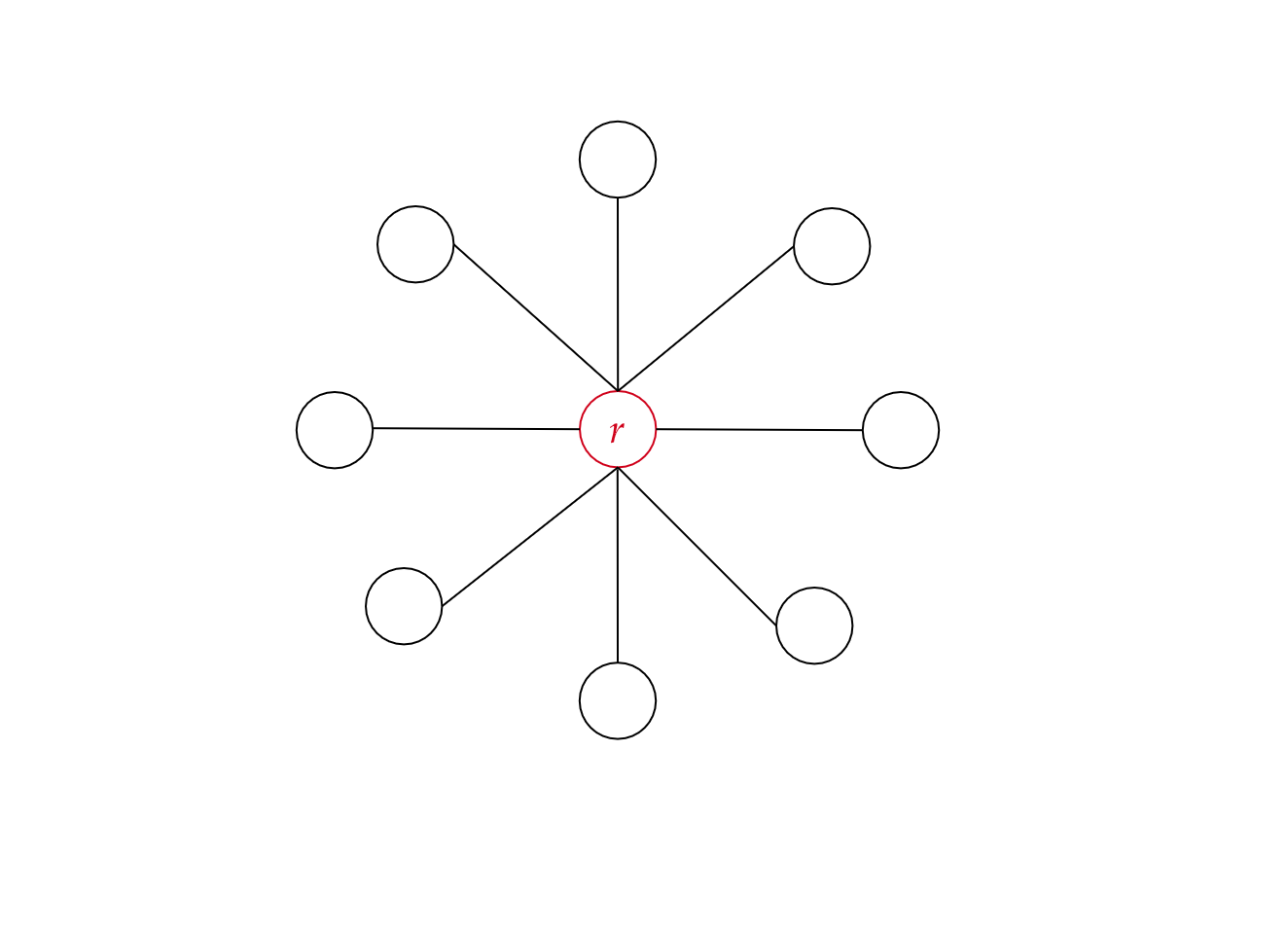} &
        \includegraphics[scale=.15]{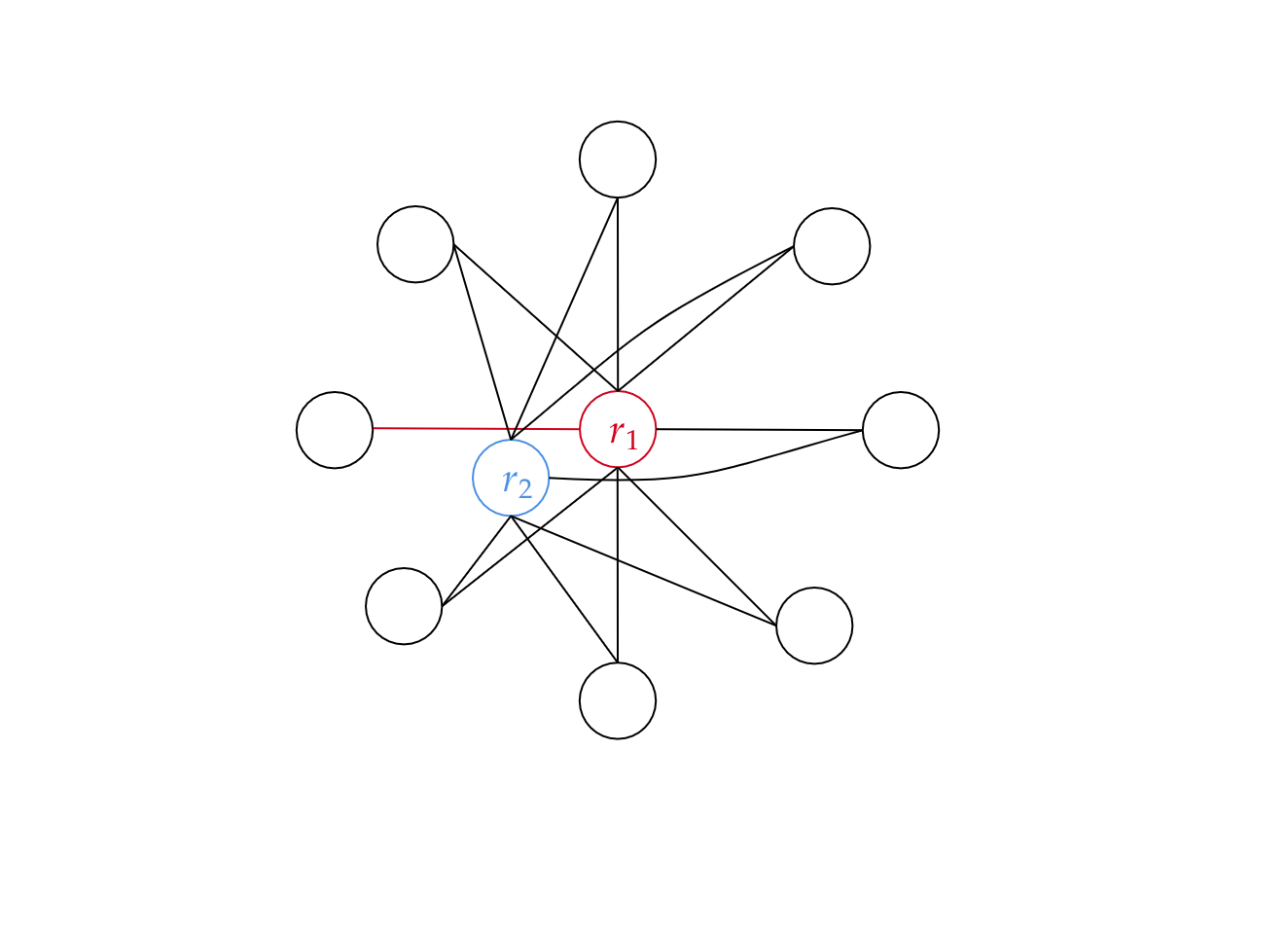}\\[-1cm]
        (a) Example 1. & (b) Example 2.
\end{tabular}
\vskip -.1in
\caption{Sub-optimality examples}
\label{fig:examples}
\end{figure}

\textbf{Sub-optimality of using the minimum dominating set
  relaxation.}  In the second problem instance, given in
Figure~\ref{fig:examples}(b), we consider a star-like graph in which
we have a revealing vertex $r_1$, adjacent to all other vertices. We
also have an "almost" revealing vertex $r_2$ which is adjacent to all
vertices but a single leaf vertex (leaves are the vertices with degree
1 and 2 in this case). The optimal arm is again chosen uniformly among
the leaves. Rewards are set so that the gap at $r_1$ is $\Delta_{max}$
and the remaining gaps are $\Delta_{min}$. The solution to the LP of
\cite{buccapatnam2014stochastic,buccapatnam2017reward} puts all the
weights on $r_1$. However, the optimal policy for this problem
consists of playing $r_2$ and the leaf vertex not adjacent to $r_2$
until all arms but the optimal arm are eliminated. The instance
optimal regret in this case is $O(\frac{\log(T)}{\Delta_{\min}})$,
while \textsc{ucb-lp} incurs regret
$\Omega(\frac{K\log(T)}{\Delta_{\min}})$.

Next, we discuss \citep{wu2015online} and
\citep{li2020stochastic}. Their instance-dependent algorithms are
based on iteratively solving empirical approximations to
\ref{eq:lp}. For simplicity, we only discuss the instance-dependent
regret bound in \citep{li2020stochastic}. A similar bound can be found
in \citep{wu2015online}. Let $x(\Delta)$ denote the solution of
\ref{eq:lp} and define the following perturbed solution
\begin{align*}
  x_i(\Delta,\epsilon) = \sup \curl*{x_i(\Delta') \colon
  |\Delta'_i - \Delta_i| \leq \e, \forall i \in [K] }.
\end{align*}
The solution $x(\Delta,\epsilon)$ is the solution of \ref{eq:lp}
with $\e$-perturbed gaps.
\citep{li2020stochastic}[Theorem~4] states that the expected regret of their
algorithm is bounded as follows:
% \begin{align*}
  $\Reg(T)
  \leq O \paren[\bigg]{\sum_{i\in[K]} \log(T)x_i(\Delta,\epsilon)\Delta_i
    + \sum_{t = 1}^T \exp\paren*{-\tfrac{\beta(t)\epsilon^2}{K}} + K},$
% \end{align*}
for any $\epsilon > 0$ and $\beta(t) = o(t)$. For the standard bandit
problem with Gaussian rewards, we can compute the perturbed solution:
$x_i(\Delta,\epsilon) = \max\paren*{\frac{1}{(\Delta_i +
  \epsilon)^2},\frac{1}{(\Delta_i - \epsilon)^2}}$. Thus, for a
meaningful regret bound, we would need $\epsilon =
\Theta(\Delta_{\min})$. If $\epsilon$ is much smaller, then the term
$\sum_{t = 1}^T \exp\paren[\big]{-\frac{\beta(t)\epsilon^2}{K}}$
becomes too large and otherwise we risk making $x_i(\Delta,\epsilon)$
too large. To analyze the second term more carefully, we allow
$\beta(t) = t$. The first $\frac{K}{\Delta_{\min}^2}$ terms of
$\sum_{t=1}^T \exp\paren[\big]{-\frac{\beta(t)\epsilon^2}{K}}$ are now
at least $\frac{1}{e}$ and thus this sum is at least
$\sum_{t = 1}^T \exp\paren[\big]{-\frac{\beta(t)\epsilon^2}{K}}
 \geq \frac{K}{\Delta_{\min}^2}$.
Thus, the bandit regret bound evaluates to at least
  $\Omega\paren[\Big]{\sum_{i\in[K]}\frac{\log(T)}{\Delta_i}
  + \frac{K}{\Delta_{\min}^2}}$.
While this bound is asymptotically optimal, since the second term
does not have dependence on $T$, it admits a very poor dependence on
the smallest gap. We can repeat the argument above with a star-graph
construction in which the revealing vertex has gap $\Delta_{\min}$. In
this case, the optimal strategy given by the solution to
\ref{eq:lp} consists of playing the revealing vertex for
$\frac{1}{\Delta_{\min}^2}$ times and incurs regret at most
$O\paren*{\frac{\log(T)}{\Delta_{\min}}}$. The regret bound of the algorithm
of \citet{li2020stochastic}, however, amounts to
$\Omega \paren[\Big]{\frac{\log(T)}{\Delta_{\min}} + \frac{K}{\Delta_{\min}^2}}$.

%%%%%%%%%%%%%%%%%%%%%%%%%%%%%%%%%%%%%%%%%%%%%%%%%%%%%%%%%%%%
\section{Instance-dependent finite-time bounds}
\label{sec:finite-time-bounds}

In this section, we provide an in-depth discussion of what finite-time
optimality actually means. 
Finite-time bounds are statements of the
form $\Reg(T)\leq f(T)$, which hold for any $T>0$.  Specifically, we are
considering functions of the type $f(T) = c^*\log(T)+d,$\footnote{When the problem parameters are clear from the context, we will write $c^*$ instead of $c^*(\Delta,G)$.} which we know
to exist from prior work.\footnote{While obtaining exact asymptotic optimality $\lim_{T\rightarrow \infty}\frac{f(T)}{\log(T)}=c^*$ would be ideal, we settle for optimality up to a multiplicative constant in our upper bounds.}
The question of what the optimal expression of $d$ might be seems easy to answer at first. Indeed, for bandits with Gaussian rewards, one can achieve $d = O(\sum_{i} \Delta_i)$, which in general is much smaller than $c^* = \Omega(\sum_{i} \frac{1}{\Delta_i})$~\citep{lattimore2020bandit}, and hence will be dominated by the time-dependent part of the regret for almost all reasonable lengths of the time horizon $T$.
In full information, we obtain a meaningful optimal value $d^*$ for a given gap vector by considering the worst-case regret of any algorithm under any permutation of the arms. This leads to $d^* = O\paren*{\frac{\ln(K)}{\Delta_{\min}}}$ \citep{mourtada2019optimality}. Note that in the full information setting we have $c^* = 0$.

One might hope for a similar structure for feedback graphs, where the optimal $d$ depends only on the ``full-information structure'', that is the gaps of arms neighboring an optimal arm.
All other arms contribute to $c^*$ and we might assume that their complexity is already captured in the $c^*\log(T)$ term as it is the case for bandits.
However, the situation is more complicated, as we show next.
\begin{theorem}
\label{thm:lower_bound_cstar}
For any $\Delta_{\min}$ and $K$, there exists a graph $G$ such that for any algorithm, there exists an
instance with a unique optimal arm, such that the regret of the
algorithm satisfies
    $\frac{\Reg(T)}{c^*\log(T)+\frac{1}{\Delta_{\min}}} = \tilde{\Omega}\paren[\big]{K^\frac{1}{8}}\,,$
for any $T\geq \Omega\paren*{\frac{K^{3/4}}{\Delta_{\min}}}$.
\end{theorem}
Theorem~\ref{thm:lower_bound_cstar} shows that there exists a problem instance, in which $d\gg c^*$ dominates the finite time regret for any $T \leq O(\exp(K^{1/8}))$. We note that this is not simply due to the full-information structure of the feedback graph $G$, as $c^*$ is positive. Furthermore, our results suggest that there is no simple characterization of $d$ in terms of $c^*$, e.g., $d= \Theta(c^*/\Delta_{\min})$. As shown in Section~\ref{sec:dstar_bounds}, there exists a non-trivial family of graphs, for which for any rewards instance, we have $d = O(c^*)$.

Having established that $d$ could be the dominating term in the regret for any reasonable time horizon, we now discuss the hardness of defining an optimal $d$.
Let us first consider a simple two-arm full-information problem and inspect algorithms of the style: ``Play arm 1, unless the cumulative reward of arm 2 exceeds that of arm 1 by a threshold of $\tau$.'' This kind of algorithm has small regret (small $d$) if arm 1 is optimal, and large otherwise. Tuning $\tau$ yields different trade-offs between the two scenarios.
The same issue appears in learning with graph feedback on a larger scale.
Given two instances defined by gap vectors $\Delta$ and $\Delta'$ respectively, an agent can trade off the constant regret part $d(\Delta)$ and $d(\Delta')$ in the two instances.
Take for example two algorithms $\Acal$ and $\Bcal$ and assume the respective values of $c^*$ and $d$ for the two instances and algorithms are given by Table~\ref{tab:toy}.
As we show in the Appendix~\ref{app:example_d}, there exist indeed a feedback graph and instances that are consistent with the table.
\begin{table}[ht]
\begin{center}
\begin{tabular}{l|c|c|c}
     &$c^*$& $d$ for Alg. $\Acal$ & $d$ for Alg. $\Bcal$\\
     \hline
    Instance $\Delta$ & $C/3$ & $ C/3$ & $4C/9$\\
    Instance $\Delta'$& $4C\varepsilon$ & $C/2$ & $4C\varepsilon$
\end{tabular}
\caption{Comparison of $c^*$ and $d$.}
\label{tab:toy}
\end{center}
\vskip -.15in
\end{table}
Which algorithm is more ``optimal'', $\Acal$ or $\Bcal$?
$\Bcal$ ensures that $\max_{\delta\in\{\Delta,\Delta'\}}\frac{d(\delta)}{c^*} = \Ocal(1)$ and we can write the regret function as $f(T)=\Ocal(c^*\log(T))$ without the need of a constant term $d$ at all. This algorithm minimizes the competitive ratio of $d$ and $c^*$.
$\Acal$ minimizes the worst-case absolute regret $\Acal = \argmin_{a\in\{\Acal,\Bcal\}}\max_{\delta\in \{\Delta,\Delta'\}} d(\delta,a)$.

Thus, we argue that the notion of optimality is subject to a choice
and that there is no unique correct answer.
In this paper, we opt for the second choice, a minimax notion of optimality, for the following reasons: 1.~Theorem~\ref{thm:lower_bound_cstar} shows that a constant competitive ratio is generally unachievable; 2.~Optimizing regret in general is a different objective than that of competitive ratio. Optimizing for a mixture implies a counter-intuitive preferences such as: ``In a hard environment where I cannot avoid suffering a loss of 1000, it does not matter much if I suffer an additional 1000 on top, as long as I do better on easier environments.'' 3. Moreover, note that, even if one were interested in optimizing the competitive ratio between $c^*$ and $d$, it is unclear if one could achieve that computationally efficiently.

We present our final definition for an optimal notion of $d^*$ in Section~\ref{sec: d lower bound}.
The high level idea is to take all confusing instances, where the means are perturbed by less than $\Delta_s$ and consider the worst-case regret any algorithm suffers over these instances until identifying all gaps up to $\Delta_s$ precision.

%%%%%%%%%%%%%%%%%%%%%%%%%%%%%%%%%%%%%%%%%%%%%%%%%%%%%%%%%%%%
\section{Algorithm and regret upper bounds}

Our algorithm works by approximating the gaps $(\Delta_i)_{i \in [K]}$
and then solving a version of \ref{eq:lp}.
First, note that all arms $i$ with gaps $\Delta_i \leq \frac{1}{T}$ can be ignored as the total contribution to the regret is at most $O(1)$. We now segment the interval $[\frac{1}{T},1]$, containing each relevant gap, into sub-intervals $[2^{-s},2^{-s+1}]$, where $s \in [\lceil\log_2(T)\rceil]$. The algorithm now proceeds in phases corresponding to each of the $\lceil\log_2(T)\rceil$ intervals. During phase $s$, all arms with gaps $\Delta_i \in [2^{-s},2^{-s+1}]$ will be observed sufficiently many times to be identified as sub-optimal. 

For phase $s$, let $\Delta_s = 2^{-s}$
denote the smallest possible gap that can be part of the interval $[2^{-s},2^{-s+1}]$ and define the clipped gap vector as $\Delta^s
\in \RR^K, \Delta^s_i = \Delta_s\lor \Delta_i$. Further, define the
set $\Gamma_s = \curl*{i \in [K]\colon \Delta_i \leq 2\Delta_s}$. $\Gamma_s$ consists of all optimal arms $I^*$ and all sub-optimal arms with gaps small enough, making them impossible to distinguish from optimal arms.
Define the following LP
\begin{equation}
\label{eq:lp2}
        \min_{x \in \RR^{[K]}} \ \tri*{\Delta^s, x } \qquad
        s.t. \ \sum_{j\in N_i} x_j \geq \frac{1}{\Delta_s^2},
        \ \forall i\in \Gamma_s\,.\tag{LP2}
\end{equation}
For any arm $i$ such that $\Delta_i \in [2^{-s},2^{-s+1}]$ observing $i$ for $\frac{1}{\Delta_s^2}$ times is sufficient to identify $i$ as a sub-optimal arm. Further, information theory dictates that $i$ needs to be observed at least $\frac{1}{\Delta_{s-1}^2}$ times to be distinguished as sub-optimal. Thus, the constraints of \ref{eq:lp2} are necessary and sufficient for identifying the sub-optimal arms $i$ with $\Delta_i \in [2^{-s},2^{-s+1}]$.
Furthermore, since there is no sufficient information to
distinguish between any two arms $i$ and $j$ with gaps $\Delta_i
\leq \Delta_j < \Delta_s$, we choose to treat all of them as equal in
the objective of the LP. Indeed, Lemma~\ref{lem:dstar_lower} shows
that for any graph $G$ and any algorithm, there exists an assignment
of the gaps $\Delta_i < \Delta_s$ so that the algorithm will suffer
regret proportional to the value of \ref{eq:lp2}.  

In practice, it is impossible to devise an algorithm that solves and plays according to \ref{eq:lp2} because even during phase $s$, there is still no
complete knowledge of the gaps $\Delta_i > \Delta_s$, but, rather only
empirical estimators, and so there is no access to $\Delta^s$.
We also replace the constraints by a confidence interval term of the order $\frac{\log(1/\delta_s)}{\Delta_s^2}$. This enables us to bound the probability of failure for the algorithm by $\delta_s$ during phase $s$.
We note that standard choices of $\delta_s$ such as $\delta_s =
\Theta\left(\frac{1}{T}\right)$ from UCB-type strategies will result in a regret bound that has a sub-optimal time-horizon dependence.
This suggests that a more careful choice of $\delta_s$
must be determined.
\vspace*{-5pt}
\subsection{Algorithm}
\vspace*{-5pt}
To describe our algorithm, we will adopt the following definitions and notation.
Let $\tau_s$ denote the last time-step of phase $s$.
We will denote by $n_i(s)$ the total number of times
the reward of arm $i$ is observed up to and including $s$, $n_i(s) = \sum_{t=1}^{\tau_s}\Ind(i_t\in N_i)$, and by $r_i(s)$
the average reward observed,
    $\hat r_i(s) = \bracket*{\sum_{t=1}^{\tau_s}r_{t,i} \Ind(i_t\in N_i)}/n_i(s)$.
We also denote by 
$\hat\Delta_i(s)$ a lower bound on $\Delta^s$ with a shrinking confidence interval $b_i(s)$ and by $\hat \Gamma_s$ the empirical version of the set $\Gamma_s$:
\begin{align*}
& \hat \Delta_{i}(s) = \Delta_s\lor\max_{j\in [K]} \hat r_j(s) - b_j(s) - \hat r_i(s) - b_i(s), \text{ where }
    b_i(s) = \sqrt{\frac{3\alpha \log(\frac{ K}{\Delta_{s+1}})}{n_i(s)}}\\
    &\hat\Gamma_s := \{i\in[K]\,|\,\hat\Delta_i(s-1)\leq 2\Delta_s\}
\end{align*}
% As indicated earlier, 
Our algorithm solves an empirical version of \eqref{eq:lp2} at each phase, which is the following LP:
\begin{equation}
\label{eq:emp_lp}
    \begin{aligned}
        \min_{x\in \mathbb{R}_{+}^K} \tri*{x,\hat\Delta(s-1)} \qquad
        s.t. & \  \sum_{j \in N_i} x_j \geq \frac{\alpha' \log(\frac{K}{\Delta_{s+1}})}{\Delta_s^2}, \forall i \in \hat \Gamma_s,\\
        &\ \sum_{j \in N_i} x_j \geq \frac{\alpha'}{\hat\Delta_i^2(s-1)}, \forall i \not \in \hat \Gamma_s\,.
    \end{aligned}\tag{LP3}
\end{equation}
Pseudocode can be found in Algorithm~\ref{alg:emp_lp}.
In the first $\lceil\log(K)\rceil$ rounds, the algorithm just plays according to the minimum dominating set of $G$. This is because there is not enough information regarding any of the gaps.
Denote the approximate solution of \ref{eq:emp_lp} as $x^*_{\ref{eq:emp_lp}}$ at phase $s$. Then at every round of
phase $s$ we play each arm exactly $\lceil (x^*_{\ref{eq:emp_lp}})_i\rceil$ many times.
Phase $s$ then ends after
$\sum_{j\in[K]}\lceil(x^*_{\ref{eq:emp_lp}})_i\rceil$ rounds. We
note that it is sufficient to approximately solve
\ref{eq:emp_lp} so that the constraints are satisfied up to some
multiplicative factor and the value of the solution is bounded by a
multiplicative factor in the value of the LP.

\begin{algorithm}[t]
\SetKwInOut{Input}{Input}
\Input{Graph $G=(V,E)$, confidence parameter $\delta$, time horizon $T$}
\textbf{Initialize $t=0$, $s=0$, $\hat r_i(0) = 0, \forall i \in [K]$}\\
Compute (approximate) minimum dominating set $\hat\Dcal(G)$\\
\While{$s \leq \lceil \log(K) \rceil$}{
    Play each arm $i\in \hat\Dcal(G)$ for $\frac{\alpha' \log(\frac{K}{\Delta_{s+1}})}{\Delta_s^2}$ rounds\\
    Update $t$ and $s$.
}
\While{$t\leq T$}{
    Compute a (approximate) solution $x^*_{\ref{eq:emp_lp}}$ to \ref{eq:emp_lp}.\\
    Play each action $i$ for $\lceil (x^*_{\ref{eq:emp_lp}})_i\rceil$ rounds and update $t$.\\
    Update the phase $s+=1$.
}
\caption{Algorithm based on \ref{eq:emp_lp}}
\label{alg:emp_lp}
\end{algorithm}
\vspace*{-5pt}
\subsection{Regret bound}
\vspace*{-5pt}
The first step in the regret analysis of Algorithm~\ref{alg:emp_lp} is
to relate the value of \ref{eq:emp_lp} to the value of \ref{eq:proxy_lp_mod} based on the
true gaps given below.
\begin{equation}
\label{eq:proxy_lp_mod}
    \begin{aligned}
       \min_{x\in \mathbb{R}_{+}^K} \tri*{x,\Delta^s} \qquad
        s.t. \sum_{j \in N_i} x_j &\geq \frac{\alpha' \log(\frac{K}{\Delta_{s+1}})}{\Delta_s^2}, \forall i \in \Gamma_s,\\
        \sum_{j \in N_i} x_j &\geq \frac{\alpha'}{\Delta_i^2}, \forall i \not \in \Gamma_s.
    \end{aligned}\tag{LP4}
\end{equation}
We do so by showing that $\hat\Gamma_{s+1}\subseteq \Gamma_s$ and that
$\hat\Delta(s) = \Theta(\Delta^s)$. This allows us to upper upper
bound the value of \ref{eq:emp_lp} by the value of
\ref{eq:proxy_lp_mod} in the following way.
\begin{lemma}
\label{lem:emp_lp_reg_bound}
Let $D_{\ref{eq:emp_lp}}(s)$ be the value of \ref{eq:emp_lp} at phase $s$ and let $D_{\ref{eq:proxy_lp_mod}}(s)$ be the value of \ref{eq:proxy_lp_mod} at phase $s$. For any $s\geq \log(K)\lor 10$ holds that
    $D_{\ref{eq:emp_lp}}(s+1) \leq 4D_{\ref{eq:proxy_lp_mod}}(s),$
with probability at least $1-3\left(\frac{\Delta_{s/2+1}}{K}\right)^{\alpha-2}$.
Further, for any $s\geq \log(|I^*|/(4\Delta_{\min}))\lor 10$
it holds that the regret incurred for playing according to \ref{eq:emp_lp} is at most $16\alpha' c^*(G,\mu)$
with the same probability.
\end{lemma}
Lemma~\ref{lem:emp_lp_reg_bound} shows that playing
Algorithm~\ref{alg:emp_lp} is already asymptotically optimal, as the
incurred regret during any phase $s\geq \log(|I^*|/\Delta_{\min})$ starts
being bounded by $O(c^*)$.  There are two challenging parts in proving
Lemma~\ref{lem:emp_lp_reg_bound}. First is how to handle the concentration of
$\hat\Delta_i(s)$ for actions $i\not\in\hat\Gamma_s$ which have been
eliminated prior to phase $s$. This challenge arises because $\alpha'$
needs to be set as a time-independent parameter as the time-horizon
part of the regret incurred by the algorithm will depend on
$\alpha'$. We notice that for any phase $s \geq 2$ the event that the empirical reward, $\hat r_i(t)$,
concentrates uniformly around its mean $\mu_i$ in the interval $t \in
[s/2,s]$ can be controlled with high probability. This in turn guarantees that the
empirical gap estimator $\hat\Delta(t)$ is small enough and hence
action $i$ is observed sufficiently many times in phases $[s/2,s]$.

The second challenge is to analyze the regret of the solution of \ref{eq:emp_lp} directly, for any $s \geq \log\left(\frac{|I^*|}{K}\right)$ so that we can bound this regret by $c^*$. The key observation is that there exists a $\hat x^*$ which is feasible (with high probability) for \ref{eq:lp} with the property that 
$\langle \hat x^*, \hat\Delta(s) \rangle \leq O(c^*)$ and further $\sum_{i\not \in I^*} x^*_{\ref{eq:emp_lp},i}\hat\Delta_i(s) \leq 2\sum_{i\not \in I^*} \hat x^*_i\hat\Delta_i(s)$. This is sufficient to conclude that $D_{\ref{eq:emp_lp}}(s) \leq O(c^*)$

Lemma~\ref{lem:emp_lp_reg_bound} can now be combined with the
observation that the constraints of \ref{eq:lp2} are a subset of
the constraints of \ref{eq:proxy_lp_mod}, up to a logarithmic
factor in $\frac{1}{\Delta_{\min}}$, to argue the following upper
regret bound.
\begin{theorem}
\label{thm:emp_lp_opt}
Let $d^*(G,\mu) = \max_{s\leq \log(|I^*|/\Delta_{\min})}
D_{\ref{eq:lp2}}(s)$. There exists an algorithm with expected regret $\Reg(T)$
bounded as
    $\Reg(T) \leq O\left(\log^2\left(\frac{1}{\Delta_{\min}}\right)d^* +
    \log(T)c^* + \gamma(G)K\log(K)\right)\,.$
\end{theorem}
We note that Algorithm~\ref{alg:emp_lp} can incur additional regret of order $O(K)$ per phase due to the rounding, $\lceil x^*_{\ref{eq:emp_lp}}\rceil$, of the solution to \ref{eq:emp_lp}. Thus its regret will only be asymptotically optimal in the setting when $\Delta_{\min} \leq O(1/K)$. To fix this minor issue, we present an algorithm with more careful rounding in Appendix~\ref{app:alg_mod}, which enjoys the regret bound of Theorem~\ref{thm:emp_lp_opt}.

%%%%%%%%%%%%%%%%%%%%%%%%%%%%%%%%%%%%%%%%%%%%%%%%%%%%%%%%%%%%
\section{Regret lower bounds}

\label{sec: d lower bound}
\paragraph{Lower bound with $d^*$.}
We are able to show the following result for any algorithm.
\begin{lemma}
\label{lem:dstar_lower}
Fix any instance $\mu$ s.t. $\mu_i \leq 1-2\Delta_s, i\in I^*$. Let $\Lambda_s(\mu)$ be the set of problem instances with means $\mu'\in\mu+[0,2\Delta^s]^k$.
Then for any algorithm, there exists an instance in $\Lambda_s(\mu)$ such that the regret is lower bounded by \ref{eq:lp2}.
\end{lemma}
Motivated by Lemma~\ref{lem:dstar_lower}, the quantity $d^*(G,\mu)$ is a meaningful definition of finite-time optimality.
We note that $d^*$ is indeed independent of the time-horizon and only depends on the topology of $G$ and the instance $\mu$. 
The result in Lemma~\ref{lem:dstar_lower} is a companion to the
upper bound in Theorem~\ref{thm:emp_lp_opt}.
It shows that for any instance $\mu$ and number of observations which are not sufficient to distinguish the arms with smallest positive gaps as sub-optimal, any algorithm will necessarily incur large regret of order $d^*$. This happens because the algorithm will not be able to distinguish $\mu$ from some environment $\mu'$ which is identical to $\mu$ except for the reward of a single arm which is only slightly perturbed.

The definition of $d^*$ as a maximum over different values of $s$ might seem surprising, as one could expect that the value of \ref{eq:lp2} strictly increases when $s$ grows, after all this is precisely what happens in the bandit setting. 
This is not the case for general graphs, where the value can also decrease between phases $s$ and $s+1$. Intuitively this happens when the approximate minimum weighted dominating set chosen by the LP's solution increases between phases.

The result in Lemma~\ref{lem:dstar_lower} has a min-max flavor in the sense that all possible instances which are close to $\mu$ are considered. It is reasonable to ask if $d^*$ can be further bounded by a favorable instance-dependent quantity. The answer to this question is complicated and certainly depends on the topology of the feedback graph as we show next.

\paragraph{Sketch of proof of Theorem~\ref{thm:lower_bound_cstar}.}

\setlength{\intextsep}{0pt}
\setlength{\columnsep}{8pt}
\begin{wrapfigure}{r}{0.3\textwidth}
    \centering
    \includegraphics[angle=270,origin=c,
      width=0.33\textwidth]{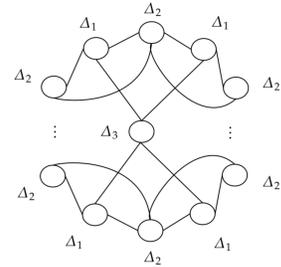}
    \vskip -.3in
    \caption{Reinforced wheel.}
    \label{fig:counter_example}
    \vskip -.05in
\end{wrapfigure}

We now show that any finite time term $d$ has to exceed $c^*$ by at
least a multiplicative polynomial factor in the number of actions
$K$. To do so we exhibit a specific feedback graph $G$, 
found in Figure~\ref{fig:counter_example}, on which any
algorithm will have to incur regret at least $\Omega(K^{\frac 1 8}
c^*)$ for some $\mu$ s.t. $c^* \geq \frac{1}{\Delta_{\min}}$. 

Formally the graph is defined to have a vertex set $V =
\Ncal_1\bigcup\Ncal_2\bigcup\Ncal_3$ of $2K+1$ arms, with each of
$\Ncal_i$'s disjoint and $\Ncal_{1}=\{2i : 1 \leq i\leq K\}$,
$\Ncal_{2} = \{2i+1: 0\leq i\leq K\}$, $\Ncal_{3} = \{0\}$. The
set of edges is defined as follows. Every vertex in $\Ncal_1$ is adjacent
to the vertex in $\Ncal_3$ and the $2i$ vertex is adjacent to both $2i
+ 1$ and $2i-1$ in $\Ncal_2$ modulo $2K+1$. Finally vertex $2i+1$ in
$\Ncal_2$ is further adjacent to to the next $\lceil K^{1/8} \rceil$
vertices in $\Ncal_2$ modulo $2K+1$. The base instance, $\mathcal{E}$, is defined by a
scalar $\nu \in [0,1]$ and gap parameter $\Delta$ so that the expected
reward of every action in $\Ncal_1$ is equal to $\nu-\Delta$, the
expected reward of every action in $\Ncal_2$ is equal to $\nu -
K^{1/4}\Delta$ and the expected reward of the action in $\Ncal_3$ is
$\nu-\sqrt{K}\Delta$. We assume that all rewards follow a Gaussian
with variance $\frac{1}{\sqrt{2}}$. We denote by $\Delta_1 = \Delta,
\Delta_2 = K^{\frac 1 4}\Delta, \Delta_3 = \sqrt{K}\Delta$.

The lower bound now fixes an algorithm $\Acal$ and considers two
cases. First, $\Acal$ could commit to often playing arms in $\Ncal_1$. In this case we show that there could be a large gap to the arms in $\Ncal_1$ which would not be detectable by $\Acal$ as arms in $\Ncal_2$ are not observed often enough. This is indeed the case as $\Acal$ needs to play $\Omega(K)$ actions in $\Ncal_1$ to cover $\Ncal_2$. This first case corresponds to assuming that the number of arms played from $\Ncal_2$ is at most $O(K^{\frac 7 8}/\Delta_2^2)$ in the first $O(K/\Delta_2^2)$ rounds.
The second case considers the scenario in which actions in $\Ncal_2$ are played for more than $\Omega(K^{\frac 7 8}/\Delta_2^2)$ times in the first $O(K/\Delta_2^2)$ rounds. In this case, $\Acal$ would suffer large regret if the gap at actions in $\Ncal_1$ is small enough, so that the optimal strategy is to cover $\Ncal_2$ by playing arms in $\Ncal_1$.

More formally, we begin by showing that there always exists an arm $n^*\in \Ncal_2$ which is observed for only $O(1/\Delta_2^2)$ times. 
Next, we change the expected reward of $n^*$ depending on which of the above two cases occur. In the first case we change the environment by setting
the reward of $n^*$ to have expectation $\nu+\Delta_2$. We can now
argue that the regret of $\Acal$ will be at least $\Omega(K^{\frac 3
  4}/\Delta)$ as $n^*$ will not be played often enough in the new
environment. The value of $c^*$, however, is at most $O(K^{\frac 7
  8}/\Delta_2) = O(K^{\frac 5 8}/\Delta)$, as playing each action in a
minimum dominating set over $\Ncal_2$ for $\frac{1}{\Delta_2^2}$
rounds is feasible for \ref{eq:lp}. For the second case, we set the
expected reward of $n^*$ to equal $\nu$. The optimal strategy now has
regret at most $O(\sqrt{K}/\Delta)$ by playing the action in $\Ncal_3$
for $\frac{1}{\Delta^2}$ and every action in $\Ncal_1$ for
$\frac{1}{\Delta_2^2}$ rounds. On the other hand $\Acal$ will incur at
least $\Omega(K^{\frac 5 8}/\Delta)$, as again $n^*$ is not played
often enough. This argument implies the result presented in
Theorem~\ref{thm:lower_bound_cstar}.

\vspace*{-7pt}
%%%%%%%%%%%%%%%%%%%%%%%%%%%%%%%%%%%%%%%%%%%%%%%%%%%%%%%%%%%%
\section{Characterizing the value of $d^*$}
\vspace*{-7pt}
\label{sec:dstar_bounds}
Theorem~\ref{thm:lower_bound_cstar} suggests that we take into account the topology of $G$ explicitly when trying to bound $d^*$, independently of the instance $\mu$. In this section, we first show a bound on $d^*$ that depends only on independent sets of $G$. Then, we show a set of graphs $G$ for which $d^* \leq O(c^*)$ on any instance $\mu$.

% \paragraph{Improving on the bound in \citep{lykouris2020feedback}.} 
Let us recall the regret bounds presented in
\citep{lykouris2020feedback}.
Denote by $\mathscr{I}(G)$ the set of all independent sets for the graph
$G$. Then the regret bounds presented in \citep{lykouris2020feedback}
are of the order
  $\Reg(T)
  \leq O\left(\max_{I \in \mathscr{I}(G)} \sum_{i\in I}\frac{\log^2(T)}{\Delta_i}\right)$.
It is possible to show, as we do in Appendix~\ref{app:thodoris_bound}, that $d^*(G,\mu) \leq \max_{I \in \mathscr{I}(G)} \sum_{i\in I}\frac{1}{\Delta_i}$
Thus, our algorithm enjoys regret bounds which are better than what is known 
for the algorithms studied in~\citep{cohen2016online,lykouris2020feedback}.
The above bound, however, could be very loose as was discussed in the beginning of the paper, especially when considering star-graphs, as the bound would just reduce to the bandit case. It turns out, however, that $d^* \leq O(c^*)$ in this case. In fact we can state a sufficient condition on $G$ so that $d^* \leq c^* + \frac{|I^*|}{\Delta_{\min}}$ for a more general family of graphs.
We begin by defining the following operation on $G$.
\begin{definition}
\label{def: collapse}
Let $\sim$ be the equivalence class defined by $u\sim v$ iff $N_u=N_v$ and
let $\mathscr{C}$ be the mapping which sends $G$ to the quotient $^{G}/_\sim$ through the operation of collapsing any sub-graph of $G$ into its equivalence class.
\end{definition}
\vspace*{-5pt}
We note that $\mathscr{C}$ is well-defined as the relation $\sim$ is an equivalent relation.
The equivalence classes defined by $\sim$ are cliques with the following property. For any $v$ in an equivalence class $[v]$ it holds that $u \in [v], \forall u \in N_v$, that is the vertices in the equivalence class clique only have neighbors in the clique to which they belong. For any instance $\mu$ of the problem, this allows us to collapse each equivalence class $[v]$ to a vertex $v$ with the maximum expected reward in $[v]$. The next lemma states a sufficient condition on $G$ under which $d^*$ is bounded.
\begin{lemma}
\label{lem:dstar_gen}
If the graph $G$ is such that $\mathscr{C}(G)$ has no path of length greater than two between any two vertices, then, for any instance $\mu$, the following inequality holds: $c^* + \frac{|I^*|}{\Delta_{\min}} \geq d^*$.
\end{lemma}
\vspace*{-10pt}
\section{Conclusion}
\vspace*{-7pt}
We presented a detailed study of the problem of stochastic online learning 
with feedback graphs in a finite time setting. We pointed out the surprising
issue of defining optimal finite-time regret for this problem.
We gave an instance on which no algorithm can hope to match, in finite time, the quantity $c^*$, which characterizes asymptotic optimality. Next, we derived an asymptotically optimal algorithm that is also min-max optimal in a finite-time sense and admits more favorable regret guarantees than those given in prior work. Finally, we described a family of feedback graphs for which matching the asymptotically optimal rate is possible in finite time.

There are several interesting questions that follow from this work. First, while the condition on $\mathscr{C}(G)$ in Lemma~\ref{lem:dstar_gen} is sufficient, it is not necessary. For example, a star-like graph in which two leaf vertices are also neighbors will have the property that $d^* \leq O(c^*)$ for any instance $\mu$.
We ask what would be a necessary and sufficient condition on $G$ for which $d^* = \Theta(c^*)$ on any instance $\mu$?
Another interesting question is how to address the setting of evolving feedback graphs. It is unclear what conditions on the graph sequence would allow us to recover bounds that improve on the existing independence number results.
Further, can we use our approach to show improved results for the setting of dependent rewards and feedback graphs studied in \citep{cortes2020online}? Finally, our methodology crucially relies on the informed setting assumption. We ask if it is possible to achieve similar bounds to Theorem~\ref{thm:emp_lp_opt} in the uninformed setting.

\newpage
\bibliographystyle{plainnat}
\bibliography{mybib}

\clearpage

%%%%%%%%%%%%%%%%%%%%%%%%%%%%%%%%%%%%%%%%%%%%%%%%%%%%%%%%%%%%

\appendix
\renewcommand{\contentsname}{Contents of Appendix}
\tableofcontents
\addtocontents{toc}{\protect\setcounter{tocdepth}{3}} 
\clearpage

\section{Refined example for hardness of determining optimal $d$}
\label{app:example_d}
\begin{figure}[t]
    \vskip -.15in
    \centering
    \includegraphics[width=0.7\textwidth]{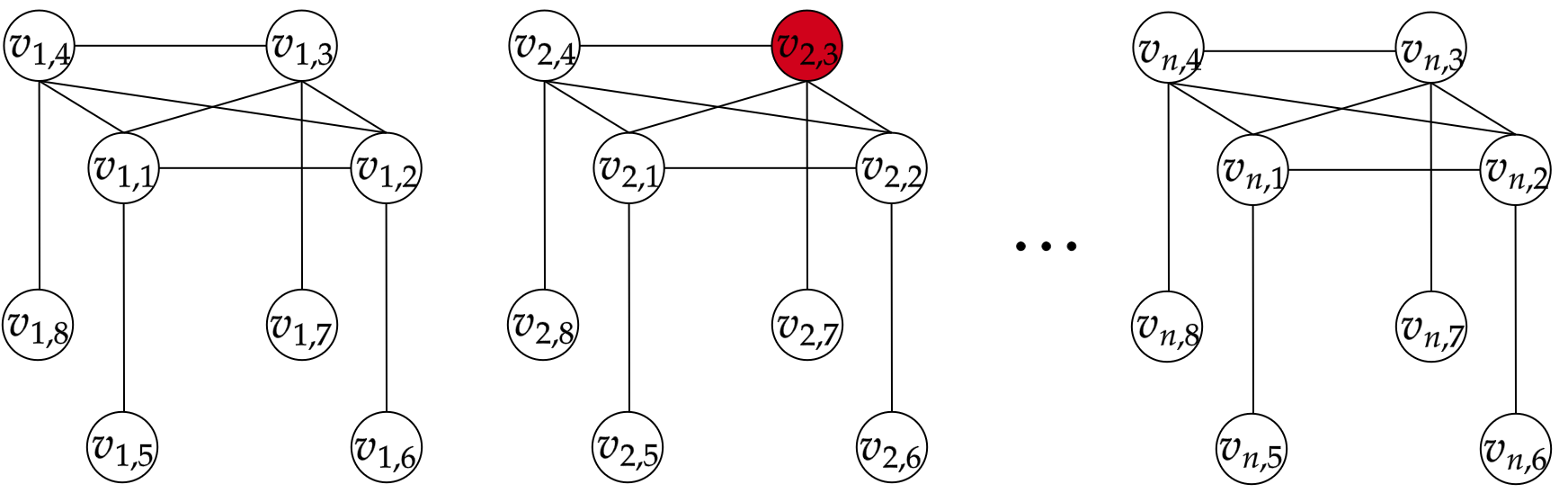}
    \caption{Arm in red is optimal}
    \vskip -.1in
    \label{fig:story_example}
\end{figure}

To understand better why it is difficult to define a notion of
\emph{optimality} for the constant term $d$ in the finite-time bound,
consider the following toy problem.  The graph is given by
Figure~\ref{fig:story_example}.  There are $n$ disjoint copies of an
open cube graph with 8 vertices each.
We let $V_1 = \{\nu_{i,1},\nu_{i,2},\nu_{i,3},\nu_{i,4}\}_{i\in[n]}$ and $V_2 =
\{\nu_{i,5},\nu_{i,6},\nu_{i,7},\nu_{i,8}\}_{i\in[n]}$.  We assume that we have oracle
knowledge of the mean rewards of all arms $\mu(\nu)=\frac{1}{2}$ for
any $\nu\in V_2$ and $\mu(\nu)=\frac{1}{2}-\Delta$ for all $\nu\in
V_1$, with one exception.  There is one arm in $V_1$, chosen uniformly
at random, that is optimal with a mean
$\mu^*\in\left\{\frac{1}{2}+2\Delta,\frac{1}{2}+\varepsilon\Delta\right\}$.
We note that we do not know the index of the optimal arm and so the 
problem reduces to identifying the optimal arm and the respective
environment (i.e.\ value of $\mu^*$). The best we can do is to collect
equally many samples for each arm in $V_1$ until we have sufficient
statistics to figure out either the environment or the optimal
arm. Under Env. A we need to collect $1/(3\Delta)^2$ samples
and under Env. B we need to collect $1/((1+\epsilon)\Delta)^2$ samples.
There are two canonical base
strategies corresponding to algorithm $\Acal$ and $\Bcal$ in Section~\ref{sec:finite-time-bounds}:  either play all arms in $V_2$ for $N(env)$ times (Algorithm $\Acal$), depending 
on the environment, or
play all arms in $V_1$ for $N(env)/4$ many times (Algorithm $\Bcal$). The following
table shows the regret each strategy suffers for collecting sufficient
samples to distinguish the environments.
\begin{center}
\begin{tabular}{l|c|c}
     & Env. A ($\mu^*=\frac{1}{2}+2\Delta$)&Env. B ($\mu^*=\frac{1}{2}+\varepsilon\Delta$)  \\
     \hline
    $\Acal$ (Play $V_1$) & $n/(3\Delta)$ & $ n/((1+\varepsilon)\Delta)$\\
    $\Bcal$  (Play $V_2$)& $ 4n/(9\Delta)$ & $4n\varepsilon/((1+\varepsilon)^2\Delta)$
\end{tabular}
\end{center}
Under Env. A we have $c^*_{\text{Env. A}} = \frac{n}{(3\Delta)}$ 
and under Env. B we have
$c^*_{\text{Env. B}} = \frac{n\varepsilon}{(1+\epsilon)^2\Delta}$. 
Which strategy is the ``optimal'' one?  One possible answer is to say
that $\Acal$ is optimal, since it minimizes the worst-case regret.
One might be
tempted to say that $\Bcal$ is better, since we can absorb the
constant term in the leading $\mathcal{O}(c^*\log(T))$ without the
need of adding a constant $d$ at all! That is, $\Bcal$ minimizes
the competitive ration.

The implicit assumption made for the second choice of optimality is:
``In a bad environment, where it is inevitable to suffer a loss of
100000, suffering an additional 100000 is just as bad as suffering an
additional loss of 10 in an environment where one cannot avoid a loss
of 10.''  We argue that this notion of optimality is not aligned with
the principle of regret as a benchmark. In regret, unlike the
competitive ratio, we care about the absolute value of
suboptimality. Hence, we claim that considering strategy 1 optimal in
our toy experiment independent of the value of $c^*$ in environment A
and B is a meaningful choice. The same argument implies that hiding
arbitrarily large constants in the $\mathcal{O}$-notation will obscure
critical information about the practicalities of an algorithm, which
our work unfortunately does as well. The regret upper bounds presented
in this work hide only universal constants which are independent of
the problem parameters, including the topology of the feedback graph.

\section{Regret upper bound proofs}

For the rest of the appendix we are going to assume that each gap $\Delta_i$ is such that $\Delta_i = 2^{-f(i)}$ for some function $f\colon [K] \to \lceil\log(T)\rceil$. This is without loss of generality as every $\Delta_i$ is in $[2^{-s}, 2^{-s+1}]$ for some $s$. Thus, we can clip every $\Delta_i$ to $2^{-s_i}$ for some $s_i$ and change the constraints and objective of \ref{eq:lp} by at most a factor of $2$. Thus the value of $c^*$ would change by at most a factor of $2$.

\subsection{Algorithm modification}
\label{app:alg_mod}

Since Algorithm~\ref{alg:emp_lp} plays $\lceil x^*_{\ref{eq:emp_lp},i}\rceil$ we need to take care of the difference $\lceil x^*_{\ref{eq:emp_lp},i}\rceil - x^*_{\ref{eq:emp_lp},i}$. At worst, playing according to the rounded solution of \ref{eq:emp_lp} can result in a $\Omega(K)$ additive factor on top of $D_{\ref{eq:emp_lp}}(s)$. This can accumulate regret up to an $\Omega(K\log(T))$ factor in the final bound. Our goal is to give asymptotically optimal bounds together with the finite time bounds and such a term might be sub-optimal in the case when $\Delta_{\min} \geq \omega(\frac{1}{K})$. 

To avoid the additional $K$-factor we modify Algorithm~\ref{alg:emp_lp} in the following way.
\begin{algorithm}[ht]
\SetKwInOut{Input}{input}
\Input{Graph $G=(V,E)$, confidence parameter $\delta$, time horizon $T$}
\textbf{Initialize $t=0$, $s=0$, $\hat r_i(0) = 0, \forall i \in [K]$, $B = [0]^K$}\\
Compute (approximate) minimum dominating set $\hat\Dcal(G)$\\
\While{$s \leq \lceil \log(K) \rceil$}{
    Play each arm $i\in \hat\Dcal(G)$ for $\frac{\alpha' \log(\frac{K}{\Delta_{s+1}})}{\Delta_s^2}$ rounds\\
    Update $t$ and $s$.
}
\While{$t\leq T$}{
    Compute a (approximate) solution $x^*_{\ref{eq:emp_lp}}$ to \ref{eq:emp_lp}.\\
    \For{$i \in [K]$}{
        \If{$x^*_{\ref{eq:emp_lp},i} < 1$}{
        \If{$B_i = 0$ or $\lfloor B_i + x^*_{\ref{eq:emp_lp},i} \rfloor \geq \lfloor B_i \rfloor$}{
            Play $i$ and update $t$
            }
            $B_i += x^*_{\ref{eq:emp_lp},i}$
        }
        \Else{
        Play $i$ for $\lceil (x^*_{\ref{eq:emp_lp}})_i\rceil$ rounds and update $t$\\
        }
    }
    Update the phase $s+=1$.
}
\caption{Modification of Algorithm~\ref{alg:emp_lp}}
\label{alg:emp_lp_mod}
\end{algorithm}
Note that for any $x^*_{\ref{eq:emp_lp},i} \geq 1$, the following inequality holds: $\lceil x^*_{\ref{eq:emp_lp},i}\rceil \leq 2x^*_{\ref{eq:emp_lp},i}$, and thus playing such arms will only increase the incurred regret by a multiplicative factor of at most $2$. Thus, we only need to consider $x^*_{\ref{eq:emp_lp}, i} < 1$. We introduce a buffer $B \in \RR^K$ which will inform us when to play an arm $i$ for which $x^*_{\ref{eq:emp_lp},i} < 1$. The first time the solution of the LP informs us to play $i$ for less than a single round, we play $i$ for a single round and update the buffer as $B_i += x^*_{\ref{eq:emp_lp},i}$. We observe that we have now overplayed $i$ and have a buffer of $1-x^*_{\ref{eq:emp_lp},i}$ extra plays of $i$. Thus, at the next phase at which $x^*_{\ref{eq:emp_lp},i} < 1$, we can check if $x^*_{\ref{eq:emp_lp},i}$ can be covered by the remaining buffer. If so, then there is no need to play arm $i$ again as we still have sufficient number of observations provided by playing $i$. If the buffer is exceeded, we again play $i$ for one round and take into account the additional overplay. Thus, at the end of phase $s$, the total number of arm $i$ has been played does not exceed 
\begin{align*}
  2\sum_{t=1}^s x^*_{\ref{eq:emp_lp},i}(t) + \lceil B \rceil \leq 2\sum_{t=1}^s x^*_{\ref{eq:emp_lp},i}(t) + s \leq 3\sum_{t=1}^s x^*_{\ref{eq:emp_lp},i}(t),
\end{align*}
where $x^*_{\ref{eq:emp_lp},i}(t)$ is the solution to \ref{eq:emp_lp} at phase $t$.
The above implies the following lemma.
\begin{lemma}
Let $x^*_{\ref{eq:emp_lp},i}(t)$ be the solution to \ref{eq:emp_lp} at phase $t$. The at the end of phase $s$ of Algorithm~\ref{alg:emp_lp_mod} the total number of plays of arm $i$ is at most
\begin{align*}
    3\sum_{t=1}^s x^*_{\ref{eq:emp_lp},i}(t).
\end{align*}
\end{lemma}
The above lemma implies that we can, at the price of a constant multiplicative factor of $3$, consider the solution of \ref{eq:emp_lp} instead of the rounded solution played by Algorithm~\ref{alg:emp_lp}. Hence, for the rest of the appendix, we do so.

\subsection{Proof of Theorem~\ref{thm:emp_lp_opt}}
We begin with a somewhat standard concentration result.
\begin{lemma}
\label{lem:martingale bound}
For any $s \in [10,\log(T)], K\geq 2,\alpha'\leq 3072$, the following inequality holds
\begin{align*}
    \PP(\exists i\in [K] : |\mu_i - \hat r_i(s)| \geq b_i(s)) \leq \left(\frac{\Delta_{s+1}}{K}\right)^{\alpha-1}\,.
\end{align*}
\end{lemma}
\begin{proof}
We use Theorem 1 from \citet{zhao2016adaptive} which states that for a sum of zero-mean, $1/2$ sub-Gaussian random variables $(X_i)_{i=1}^t$ the following inequality holds
\begin{align*}
    \PP\left(\exists t: \sum_{i=1}^t X_i \geq \sqrt{t(2\log \log_2(t) + \log(1/\delta))}\right) \leq 2\delta.
\end{align*}
We begin by bounding $\PP(|\mu_i - \hat r_i(s)| \geq b_i(s))$ for a fixed $i\in[K]$. Since action $i$ is observed at most $\frac{\alpha' \log(K/\Delta_{s+1})}{\Delta_{s+1}^2}$ 
times up to and including phase $s$, we can write
\begin{align*}
    \PP(|\mu_i - \hat r_i(s)| \geq b_i(s)) &\leq 2\PP\left(\exists t \in \left[\frac{\alpha' \log(K/\Delta_{s+1})}{\Delta_{s+1}^2}\right] : \sum_{\ell=1}^t (r_{t,i} - \mu_i) \geq \sqrt{3\alpha t \log\left(\frac{K}{\Delta_{s+1}}\right)}\right)\\
    &\leq 2\left(\frac{\Delta_{s+1}}{K}\right)^\alpha,
\end{align*}
where  we used the fact that for $s \geq 7 , K\geq 2, \alpha' \leq 512$ the following inequality holds
\begin{align*}
    \log_2\left(\frac{\alpha'\log(K/\Delta_{s+1})}{\Delta_{s+1}^2}\right) \leq \frac{K}{\Delta_{s+1}}.
\end{align*}
A union bound over $i\in[K]$ completes the proof.
\end{proof}

\begin{lemma}
\label{lem:emp_gap_lower}
Under the same assumptions as in Lemma~\ref{lem:martingale bound}, we have
\begin{align*}
    \PP(\exists i \in [K], t\in [s/2,s] : |\mu_i - \hat r_i(t)| \geq b_i(t)) \leq \left(\frac{\Delta_{s/2+1}}{K}\right)^{\alpha-2}\,.
\end{align*}
Furthermore, if we let $\Ecal_{upper} = \{\forall i\in[K], t\in [s/2,s]: \hat \Delta_i(t) \leq \Delta_i\lor\Delta_{t}\}$ then $\PP(\Ecal_{upper}) \geq 1-\left(\frac{\Delta_{s/2+1}}{K}\right)^{\alpha-2}$.
\end{lemma}
\begin{proof}
A union bound over Lemma~\ref{lem:martingale bound}, together with picking $\alpha$ sufficiently large imply
\begin{align*}
    \PP(\exists i \in [K], t\in [s/2,s] : |\mu_i - \hat r_i(t)| \geq b_i(t)) \leq \sum_{t = s/2}^s \left(\frac{1}{2^{t+1}K}\right)^{\alpha-1} \leq \left(\frac{\Delta_{s/2+1}}{K}\right)^{\alpha-2}\,.
\end{align*}
For the second part of the lemma, we assume WLOG $\Delta_i \geq \Delta_{s/2}$. Thus, we have
\begin{align*}
    \hat\Delta_i(t) = \max_{j\in[k]} \hat r_j(t) - b_{j}(t) - \hat r_i(t) - b_i(t) =: \hat r_{i^*_t}(t) - b_{i^*_t}(t) - \hat r_i (t) - b_i(t).
\end{align*}
On the event that $\{\forall i \in [K], t\in [s/2,s] : |\mu_i - \hat r_i(t)| \leq b_i(t)\}$ we have $\hat r_{i^*_t}(t) - b_{i^*_t}(t) \leq \mu_{i^*_t}$ and $\hat r_i (t) + b_i(t) \geq \mu_i$. This implies that w.p. $1-\left(\frac{\Delta_{s/2+1}}{K}\right)^{\alpha-2}$ we have 
\begin{align*}
    \hat\Delta_i(t) \leq \mu_{i^*_t} - \mu_i \leq \mu_{i^*} - \mu_i = \Delta_i,
\end{align*}
for all $i \in [K]$ and $t \in [s/2,s]$.
\end{proof}

\begin{lemma}
\label{lem:bonus_bound}
Let $\alpha' = 256\alpha$ in \ref{eq:emp_lp}. On the event $\Ecal_{upper}$ it holds that $b_i(s)\leq \frac{\Delta_{i}\lor \Delta_s}{8},\forall i\in [K]$. Thus, for any $\alpha\geq 3$ and any $\log(K)\lor 10\leq s\leq\log(T)$ the following inequality holds 
\begin{align*}
    \PP\left(\exists i\in[K]: b_i(s) \geq \frac{\Delta_i\lor\Delta_s}{8}\right) \leq \left(\frac{\Delta_{s/2+1}}{K}\right)^{\alpha-2}\,.
\end{align*}
\end{lemma}
\begin{proof}
Recall that $b_i(s) = \sqrt{\frac{3\alpha \log(\frac{ K}{\Delta_{s+1}})}{n_i(s)}}$ so we are going to bound $n_i(s)$ from below. Assume that $\Delta_i \geq \Delta_s$, the other case is handled similarly. Let $s_i$ be the phase at which $\Delta_{s_i} = \Delta_i$. On the event $\Ecal_{upper}$ we know that $\hat\Delta_i(t) \leq \Delta_i$ for all $t\in[s_i,s],i\in[K]$. The constraints in \ref{eq:emp_lp} imply
\begin{align*}
    n_i(s) \geq \sum_{t=s_i+1}^s \frac{\alpha'}{\hat\Delta_i^2(t)} + \frac{\alpha'\log(\frac K \Delta_{s_i+1})}{\Delta_{s_i}^2} \geq \frac{\alpha'(s-s_i)}{\Delta_i^2} + \frac{\alpha'\log(\frac K \Delta_{s_i+1})}{\Delta_{s_i}^2}\,.
\end{align*}
The above implies that 
\begin{align*}
    b_{i}^2(s) &\leq \frac{\alpha \log(\frac K \Delta_{s+1})}{\frac{\alpha'\log(\frac K \Delta_{s_i+1})}{\Delta_{s_i}^2} + \frac{\alpha'(s-s_i)}{\Delta_i^2}} \leq \frac{6(s+1)\alpha}{\frac{\alpha'(s_i+1)}{\Delta_i^2} + \frac{\alpha'(s-s_i)}{\Delta_i^2} + \frac{\alpha'\log(K)}{\Delta_i^2}}
\end{align*}
For $\alpha' = 768\alpha$ the above implies $b_{i}(s) \leq \frac{\Delta_i}{8}$.
\end{proof}

\begin{lemma}
\label{lem:emp_gap_interval}
Let $\Ecal_{gap}(s) = \{\forall i\in [k]: \frac{\Delta_i}{2}\lor\Delta_s \leq \hat\Delta_i(s) \leq \Delta_i\lor\Delta_s\}$. Under the assumptions of Lemma~\ref{lem:bonus_bound} we have 
\begin{align*}
    \PP\left(\Ecal_{\gap}(s)\right) \geq 1- 3\left(\frac{\Delta_{s/2+1}}{K}\right)^{\alpha-2}.
\end{align*}
\end{lemma}
\begin{proof}
If $s$ is such that $\Delta_i \leq \Delta_s$ the statement of the lemma holds from Lemma~\ref{lem:emp_gap_lower} and the definition of $\hat\Delta_i(t)$. We now consider the case $\Delta_i \geq \Delta_s$ and assume that $\Ecal_{upper}$ holds. Lemma~\ref{lem:bonus_bound} now implies that $b_{i}(s)\leq \frac{\Delta_i}{8}$. Further, assume that $|\mu_i - \hat r_i(s)| \leq b_i(s),\forall i\in[K]$.  We have
\begin{align*}
    \hat\Delta_i(s) &= \max_{j\in[K]} \hat r_{j}(s) - b_j(s) - \hat r_i(s) - b_i(s) \geq \hat r_{i^*}(s) - b_{i^*}(s) - \hat r_i(s) - b_i(s)\\
    &\geq \Delta_i - 2(b_{i^*}(s) + b_{i}(s)) \geq \frac{\Delta_i}{2}.
\end{align*}
Our assumptions fail with probability at most $\left(\frac{\Delta_{s/2+1}}{K}\right)^{\alpha-2} + 2\left(\frac{\Delta_{s+1}}{K}\right)^{\alpha-1}$.
\end{proof}

\begin{lemma}
\label{lem:constr_incl}
For any phase $s\geq \log(K)\lor 10$, it holds that $\mathbb{P}(\hat \Gamma_{s+1} \not\subseteq \Gamma_s) \leq 3\left(\frac{\Delta_{s/4+1}}{K}\right)^{\alpha-2}$.
\end{lemma}
\begin{proof}
We have $\PP(\hat \Gamma_{s+1}\not\subseteq \Gamma_s) \leq \sum_{i\not\in \Gamma_s}\PP(i\in\hat \Gamma_{s+1})$. The fact $i\not\in\Gamma_s$ implies that $\Delta_i \geq 2\Delta_s$. The result now follows by Lemma~\ref{lem:emp_gap_interval}.
\end{proof}

\begin{lemma}[Lemma~\ref{lem:emp_lp_reg_bound}]
% \label{lem:emp_lp_reg_bound}
Let $D_{\ref{eq:emp_lp}}(s)$ be the value of \ref{eq:emp_lp} at phase $s$ and let $D_{\ref{eq:proxy_lp_mod}}(s)$ be the value of \ref{eq:proxy_lp_mod} at phase $s$. For any $s\geq \log(K)\lor 10$ the following inequality holds 
\begin{align*}
    D_{\ref{eq:emp_lp}}(s+1) \leq 4D_{\ref{eq:proxy_lp_mod}}(s),
\end{align*}
with probability at least $1-3\left(\frac{\Delta_{s/2+1}}{K}\right)^{\alpha-2}$.
Further, for any $s\geq \log(K)\lor 10 \lor \log\left(\frac{|I^*|}{\Delta_{\min}}\right)$ 
we have that the regret incurred for playing according to \ref{eq:emp_lp} is at most $16\alpha' c^*(G,\mu)$
with probability at least $1-3\left(\frac{\Delta_{s/2+1}}{K}\right)^{\alpha-2}$.
\end{lemma}
\begin{proof}[Proof of Lemma~\ref{lem:emp_lp_reg_bound}]
For any $s$ and all $i\in[K]$ Lemma~\ref{lem:emp_gap_interval} and Lemma~\ref{lem:constr_incl} imply that $\hat \Gamma_{s+1} \subseteq \Gamma_s$ and $\frac{\Delta^s}{2}\leq \hat\Delta(s) \leq \Delta^s$ with probability at least $1-3\left(\frac{\Delta_{s/2+1}}{K}\right)^{\alpha-2}$. If we let $x^*_{\ref{eq:proxy_lp_mod}}(s)$ be a solution to \ref{eq:proxy_lp_mod} at phase $s$, then these conditions imply that $4x^*_{\ref{eq:proxy_lp_mod}}(s)$ is feasible for \ref{eq:emp_lp}. This implies 
\begin{align*}
    D_{\ref{eq:emp_lp}}(s+1) \leq 4\langle x^*_{\ref{eq:proxy_lp_mod}}(s), \hat\Delta(s) \rangle \leq 4\langle x^*_{\ref{eq:proxy_lp_mod}}(s), \Delta^s \rangle = 4D_{\ref{eq:proxy_lp_mod}}(s).
\end{align*}
Further, for $s\geq \log(1/\Delta_{\min})$, $\Gamma_s$ consists only of $I^*$. Let $x^*_{\ref{eq:emp_lp}}$ be a solution to \ref{eq:emp_lp}, and let $\hat x^*$ be a solution to the LP dropping all constraints on $I^*$ and its neighborhood.
Note that $\tri{\hat x^*, \hat\Delta(s)} \leq 8\alpha' c^*(G,\mu)$ under $\Ecal(s)$.
We show by contradiction that 
\begin{align*}
    \sum_{i\not\in I^*}x^*_{\ref{eq:emp_lp},i}\hat \Delta_i(s) \leq 2\sum_{i\not\in I^*}\hat x_i^*\hat \Delta_i(s)\,,
\end{align*}
which by $\Delta_i \leq 2 \hat\Delta_i(s)$ completes the proof.
Assume the opposite is true, take a new $x$ such that
\begin{align*}
    &x_i = \hat x_i^* \,\forall i\not\in I^*\\
    &x_i = x^*_{\ref{eq:emp_lp},i} + \sum_{j\not\in I^*}x^*_{\ref{eq:emp_lp},j}\,\forall i\in I^*\,.
\end{align*}
$x$ is a feasible solution of \ref{eq:emp_lp}. Next, we only consider $s\geq \log(|I^*|/(4\Delta_{\min}))$ which implies that
\begin{align*}
    \tri{x,\hat\Delta^s(s)} &= \sum_{i\in I^*} x^*_{\ref{eq:emp_lp},i}\Delta_s
    +\sum_{i\not\in I^*}(x^*_{\ref{eq:emp_lp},i}|I^*|\Delta_s+\hat x_i^* \hat\Delta_i(s))\\
    &< \sum_{i\in I^*} x^*_{\ref{eq:emp_lp},i}\Delta_s
    +\sum_{i\not\in I^*}(x^*_{\ref{eq:emp_lp},i}(\frac{\Delta_{\min}}{4}+\frac{ \hat\Delta_i(s)}{2}) \\
    &\leq D_{\ref{eq:emp_lp}}(s)\,, 
\end{align*}
which is a contradiction to $\tri{x,\hat\Delta^s(s)}\geq D_{\ref{eq:emp_lp}}(s)$.

\end{proof}

Denote the value of \ref{eq:lp2} at phase $s$ as $D_{\ref{eq:lp2}}(s)$ and a solution to the LP as $x^*_{\ref{eq:lp2}}(s)$. We note that for any $s$ it holds that $x_{\ref{eq:proxy_lp_mod}}(s) = (\alpha\log(K/\Delta_{s+1})\lor \alpha')\sum_{t\leq s}x^*_{\ref{eq:lp2}}(t)$ is feasible for \ref{eq:proxy_lp_mod}. Further we have that for all $t\leq s$ it holds that $\langle x,\Delta^s \rangle \leq \langle x, \Delta^t \rangle$. These two observations imply
\begin{align*}
    D_{\ref{eq:proxy_lp_mod}}(s) \leq (\alpha\log(K/\Delta_{s+1})\lor \alpha')\sum_{t\leq s}D_{\ref{eq:lp2}}(t)\,.
\end{align*}
Further, we have $D_{\ref{eq:lp2}}(t) \leq D_{\ref{eq:proxy_lp_mod}}(s)$. We can assume that $s\geq 10$, otherwise the regret is $O(K)$. Thus we can characterize the optimality of Algorithm~\ref{alg:emp_lp} up to factors of $\log^2(1/\Delta_{\min})$ as follows.
\begin{theorem}[Theorem~\ref{thm:emp_lp_opt}]
Let $d^*(G,\mu) = \max_{s\leq \log(|I^*|/\Delta_{\min})} D_{\ref{eq:lp2}}(s)$. The expected regret $R(T)$ of playing according to Algorithm~\ref{alg:emp_lp_mod} with $\alpha = 4$ and $\alpha' = 768\alpha$ is bounded as
\begin{align*}
    R(T) \leq O\left(\log^2\left(\frac{1}{\Delta_{\min}}\right)d^*(G,\mu) + \log(T)c^*(G,\mu) + \gamma(G)K\right)\,.
\end{align*}
Further, for any algorithm, there exists an environment on which the expected regret of the algorithm is at least $\Omega(d^*(G,\mu))$.
\end{theorem}
\begin{proof}
Lemma~\ref{lem:emp_lp_reg_bound} implies that the regret bounds fail to hold at any phase $s\geq \log(K)$ w.p. at most $3\left(\frac{1}{2^{s/2+1}K}\right)^{\alpha-2}$. Further the regret at phase $s$ is always bounded by $\alpha' K2^{s}\log(1/\Delta_{min})$
Choosing $\alpha = 4$ implies expected regret of only $O(\log(1/\Delta_{\min}))$ on the union bound of failure events. For the remainder of the proof we now have for $s\leq \log(|I^*|/\Delta_{\min})$
\begin{align*}
    D_{\ref{eq:emp_lp}}(s) \leq 4D_{\ref{eq:proxy_lp_mod}}(s) \leq O(\log(K/\Delta_{s+1})d^*(G,\mu))\,.
\end{align*}
For $s\geq \log(|I^*|/\Delta_{\min})$ we have that
\begin{align*}
     D_{\ref{eq:emp_lp}}(s) \leq O(c^*(G,\mu))\,.
\end{align*}
Finally the regret incurred in the first $s\leq\log(K)$ phases is at most $O(\gamma(G)K\log(K))$ as the algorithm plays the approximate solution corresponding to the minimum dominating set of $G$. Combining all of the above shows the regret upper bound. The regret lower bound follows from Lemma~\ref{lem:dstar_lower}.
\end{proof}
\section{Regret lower bounds}
\subsection{Proof of Lemma~\ref{lem:dstar_lower}}
\begin{lemma}[Lemma~\ref{lem:dstar_lower}]
% \label{lem:dstar_lower}
Fix any instance $\mu$ s.t. $\mu_i \leq 1-2\Delta_s, i\in I^*$. Let $\Lambda_s(\mu)$ be the set of problem instances with means $\mu'\in\mu+[0,2\Delta^s]^k$.
Then for any algorithm, there exists an instance in $\Lambda_s(\mu)$ such that the regret is lower bounded by \ref{eq:lp2}.
\end{lemma}
\begin{proof}
We take as a base environment the instance with expected rewards vector $\mu$ and assume that the rewards follow a Gaussian with variance $\frac{1}{\sqrt{2}}$. 
Let $\tau \in \mathbb{R}_{+}\bigcup\{\infty\}$ be the time at which the following is satisfied
\begin{align*}
    \min_{i\in\Gamma_s}\EE\left[\sum_{t=1}^\tau \PP[A_t\in N_i]\right] = \frac{1}{4\sqrt{2}(\Delta_s)^2}\,,
\end{align*}
where the expectation is with respect to the randomness of the sampling of the rewards and $\Acal$. 

First we argue that we can assume $\tau < \infty$. Consider $\tau =\infty$. Let
\begin{align*}
    i^* = \argmin_{i\in\Gamma_s}\EE\left[\sum_{t=1}^\tau \PP_{\mu}[A_t\in N_i]\right].
\end{align*}
Fix a time horizon $T$ and let $x_{T}$ be the vector of expected number of observations of $\Acal$ on environment $\mu$ and $X_{T}$ the random vector of actual observations. By the assumption that $\tau = \infty$ and Markov's inequality, we have that $\PP[X_{T,i^*} \geq \frac{1}{\Delta_s^2}] \leq \frac{1}{2}$. Consider the algorithm $\bar \Acal$ which after $\frac{1}{\Delta_s^2}$ observations of $i^*$ switches to playing uniformly at random from $[K]\setminus N_{i^*}$ so that it never observes $i^*$ again. Let $\mu'$ be the instance which changes the expected reward of $\mu_{i^*}$ to $\mu'_{i^*} = \mu_{i^*} + 2\Delta_s$, and so $\mu' \in \Lambda_s(\mu)$. The KL-divergence between the measures induced by playing $\Acal$ on these two instances for $\tau$ rounds is bounded as $4\Delta_{s}^2 x_{\tau,i^*}$. If we let $x'_{T}$ denote the vector of expected number of observations under environment $\mu'$ then Pinsker's inequality implies that
\begin{align*}
    x'_{T,i^*} \leq \frac{1}{4\sqrt{2}\Delta_s^2}2\Delta_s\sqrt{x_{\tau,i}/2} \leq \frac{1}{4\Delta_s^2}.
\end{align*}
Thus, the expected regret of $\bar \Acal$ under $\mu'$ is at least $(T-\frac{1}{4\Delta_s^2})\Delta_s$ for any $T \geq \frac{1}{4\Delta_s^2}$. Further, by Pinsker's inequality the probability that $X_{T,i^*}\geq\frac{1}{\Delta_s^2}$ under $\bar \Acal$ in environment $\mu'$ is bounded by $\frac{3}{4}$. Since $\Acal$ and $\bar \Acal$ act in the same way up to $\frac{1}{\Delta_s^2}$ observations of $i^*$ it holds that the expected regret of $\Acal$ in environment $\mu'$ is at least $\frac{1}{4}(T-\frac{1}{4\Delta_s^2})\Delta_s$ for any $T\geq \frac{1}{4\Delta_s^2}$. Thus for $T$ large enough, e.g., $T = \Omega(D_{\ref{eq:lp2}(s)}/\Delta_s + \frac{1}{4\Delta_s^2})$ the conclusion of the lemma holds.

We now assume that $\tau \leq \infty$. Let $x_{\tau,i} = \EE\left[\sum_{t=1}^\tau \PP[A_t\in N_i]\right]$ be the expected number of observations of action $i$ after $\tau$ rounds. Assume that $\sum_{i\not\in\Gamma_s} x_{\tau,i}\Delta_i \leq \frac{D_{\ref{eq:lp2}}(s)}{16}$, otherwise we are done. The definition of $\tau$ with the above assumption imply that 
\begin{align*}
   x_{\tau,i} &\geq \frac{1}{4\sqrt{2}\Delta_s^2},\forall i\in\Gamma_s\\
    &\implies\\
    \sum_{i\in[K]} x_{\tau,i}\Delta^s_i &\geq \frac{\bar D(\Delta^s,\Gamma_s)}{4\sqrt{2}}\\
    &\implies\\
    \tau\Delta_s &\geq \sum_{i\in \Gamma_s}x_{\tau,i}\Delta_s \geq \frac{\bar D(\Delta^s,\Gamma_s)}{16}.
\end{align*}
Let $\mu'$ be the instance which changes the expected reward of $\mu_{i^*}$ to $\mu'_{i^*} = \mu_{i^*} + 2\Delta_s$, and so $\mu' \in \Lambda_s(\mu)$. The KL-divergence between the measures induced by playing $\Acal$ on these two instances for $\tau$ rounds is bounded as $4\Delta_{s}^2 x_{\tau,i^*}$. If we let $x'_{\tau}$ denote the vector of expected number of observations under environment $\mu'$ then Pinsker's inequality implies that
\begin{align*}
    x'_{\tau,i^*} \leq \tau\Delta_s\sqrt{x_{\tau,i}/2} \leq \frac{\tau}{2}.
\end{align*}
This implies
\begin{align*}
    \sum_{i\in [K]\setminus\{i^*\}} x'_{\tau,i}\Delta^s_{i} \geq \frac{\tau}{2}\Delta_s \geq \frac{\bar D(\Delta^s,\Gamma_s)}{32}.
\end{align*}
\end{proof}

\subsection{Proof of Theorem~\ref{thm:lower_bound_cstar}}
\begin{theorem}
There exists a feedback graph $G$, with $K\geq 32$ vertices, such that for any algorithm $\Acal$ there exists an environment $\mu$ on which $R(T) \geq \Omega(K^{1/8}c^*(G,\mu))$.
\end{theorem}
\begin{proof}
For any algorithm $\mathcal{A}$, define the algorithm $\overline{\mathcal{A}}$ as follows:
If there have been more than $\frac{K^\frac{7}{8}}{64\Delta_2^2}$ pulls of actions in $\mathcal N_2$, then commit to action $\mathcal N_3$ until end of time. We call the random time-step where $\mathcal{A}$ and $\overline{\mathcal{A}}$ deviate in trajectory as $\tau$.
Define the stopping times
\begin{align*}
    T_1 &:= \min\left\{t\in \mathbb{N}\cup\{\infty\}\,|\,\mathbb{P}[\tau \leq t]>\frac{1}{2}\right\}\\
    T &= \min\left\{\left\lceil\frac{K}{128\Delta_2^2}\right\rceil, T_1 \right\}\,.
\end{align*}
Let $n^*$ be the node in $\mathcal{N}_2$ with the smallest number of expected observations at time $T$ under algorithm $\overline{\mathcal{A}}$.
Let $N_i$ denote the number of times an action in $\Ncal_i$ has been played by $\overline{\mathcal{A}}$. The total number of observations over all actions in $\mathcal{N}_2$ is
\[
2N_1 + K^\frac{1}{8} N_2 \leq \frac{K}{64\Delta_2^2}+\frac{K}{64\Delta_2^2}=\frac{K}{32\Delta_2^2}\,.
\]
Hence the number of observations of $n^*$ is bounded by $\frac{1}{32\Delta_2^2}$.
Consider the environment $\mathcal{E}_1$, where all we change is increasing the reward of $n^*$ by up to $2\Delta_2$.
By Pinsker's inequality we have
\begin{align*}
    |\mathbb{P}_{\overline{\mathcal{A}},\mathcal{E}}(E) - \mathbb{P}_{\overline{\mathcal{A}},\mathcal{E}_1}(E)| \leq \sqrt{\frac{1}{2}(2\Delta_2)^2 \frac{1}{32\Delta_2^2}} = \frac{1}{4}\,,
\end{align*}
as the largest difference in probability of any event under the two environments. We consider two possible cases below.

\paragraph{Case 1 $T < T_1$.}
Set the reward of $n^*$ to $\nu+\Delta_2$.
Define the following event:
\begin{align*}
    E := \left\{\tau > T \,\land\, N_{n^*} \leq \frac{1}{4\Delta_2^2} \right\}\,.
\end{align*}
In the first environment, we have
\begin{align*}
    \mathbb{P}(E) = 1- \mathbb{P}(E^C) \geq 1- \mathbb{P}[\tau < T]-\mathbb{P}\left[N_{n^*} > \frac{1}{4\Delta_2^2}\right]
    \geq \frac{3}{8}\,.
\end{align*}
Hence the probability of $E$ is at least $\frac{1}{8}$ in the changed environment. The regret of $\mathcal{A}$ is at least
\begin{align*}
    \frac{1}{8}\left(T-\frac{1}{4\Delta_2^2}\right)\Delta_2 >\frac{K}{1024 \Delta_2}=\Omega\left(\frac{K^\frac{3}{4}}{\Delta}\right)\,. 
\end{align*}
However, the value of \ref{eq:lp} for this environment is $\Theta\left(\frac{K^\frac{7}{8}}{\Delta_2}\right)=\Theta\left(\frac{K^\frac{5}{8}}{\Delta}\right)$.

\paragraph{Case 2 $T=T_1$}
Set the reward of $n^*$ to $\nu$.
Define the following event:
\begin{align*}
E := \left\{\tau \leq T \,\land\, N_{n^*} \leq \frac{1}{4\Delta_2^2} \right\}\,.
\end{align*}
In the base environment, we have
\begin{align*}
    \mathbb{P}(E) = 1- \mathbb{P}(E^C) \geq 1- \mathbb{P}[\tau > T]-\mathbb{P}[N_{n^*} > \frac{1}{4\Delta_2^2} \}] \geq \frac{3}{8}\,.
\end{align*}
Hence the probability of $E$ is at least $\frac{1}{8}$ in the changed environment. The regret of $\mathcal{A}$ is at least
\begin{align*}
    \frac{1}{8}\left(\frac{K^\frac{7}{8}}{64\Delta_2^2}-\frac{1}{4\Delta_2^2}\right)\Delta_2 >\frac{K^\frac{7}{8}}{1024 \Delta_2}=\Omega\left(\frac{K^\frac{5}{8}}{\Delta}\right)\,. 
\end{align*}
However, the value of \ref{eq:lp} for this environment is $\Theta\left(\frac{K^\frac{1}{2}}{\Delta}\right)$.

Hence for any algorithm, there exist an environment and time step $T=\mathcal{O}(\frac{K^\frac{1}{2}}{\Delta^2})$, such that the algorithm suffers a regret that is a factor $K^\frac{1}{8}$ larger than $c^*$.
\end{proof}

\section{Characterizing $d^*$}

\subsection{Improving on bound in \citet{lykouris2020feedback}}
\label{app:thodoris_bound}
We now show that $d^*(G,\mu) \leq
\max_{I \in \mathscr{I}(G)} \sum_{i\in I}\frac{1}{\Delta_i}$:
\begin{align*}
    & D_{\ref{eq:lp2}}(s) \leq \frac{\gamma(\Gamma_s)}{\Delta_s}
    \leq \frac{\alpha(\Gamma_s)}{\Delta_s} \leq \sum_{i\in
      \Ical(\Gamma_s)}\frac{1}{\Delta_i} \leq \max_{I \in
      \mathscr{I}(\Gamma_s)} \sum_{i\in I}\frac{1}{\Delta_i} \leq \max_{I \in
      \mathscr{I}(G)} \sum_{i\in I} \frac{1}{\Delta_i}\\
      \implies & d^*(G,\mu) \leq
    \max_{I \in \mathscr{I}(G)} \sum_{i\in I}\frac{1}{\Delta_i}.
\end{align*}
The first inequality follows from the definition of the LP, the second
inequality follows from the fact that the domination number is no
larger than the independence number, the third inequality follows from
the fact that for any $i\in \Gamma_s$ we have $\Delta_s \geq
\Delta_i$, and the fifth inequality holds by the fact that
$\mathscr{I}(\Gamma_s) \subseteq \mathscr{I}(G)$.

\subsection{Bound on $d^*$ for star-graphs}
\begin{lemma}
\label{lem:dstar_upper_star_graphs}
For the star-graph $G$ and any instance $\mu$,
the following inequality holds:  $c^* +
\frac{|I^*|}{\Delta_{\min}}\geq d^*$.
\end{lemma}
\begin{proof}
Consider the dual of \ref{eq:lp} given below
\begin{equation}
\label{eq:lp5_dual}
    \begin{aligned}
        \max_{y \in \RR^{K}}\frac{1}{\Delta_s^2}\sum_{i\in
          \Gamma_s}&y_i\\ s.t.\sum_{j\in N_i\bigcap \Gamma_s} &y_j
        \leq \Delta_s,\forall i\in \Gamma_s,\\ \sum_{j\in N_i \bigcap
          \Gamma_s} &y_j \leq \Delta_i, \forall i \in
        \Gamma_s^\mathsf{C}.
    \end{aligned}\tag{LP5}
\end{equation}
Note that for any $i\in [K]$ we can take the intersection of $N_i$
with $\Gamma_s$ as no action $j\in \Gamma^\mathsf{C}$ can increase the
value of the objective of \ref{eq:lp5_dual}.  The analysis is split
into two parts.  First consider all phases $s$ for which it holds that
$\Delta_r \leq \Delta_s$. We argue that the solution to
\ref{eq:lp2} for these phases is to just play the revealing vertex
for $\frac{1}{\Delta_s^2}$ times. Indeed we can just set $y_r =
\Delta_s$ and observe that this is feasible for the dual LP with value
$\frac{1}{\Delta_s}$. Further, setting $x_r = \Delta_s$ in the primal
also yields a value of $\frac{1}{\Delta_s}$. The fact that $\Delta_r
\geq \Delta_{\min}$ together with the lower bound of $R(T) \geq
\Omega(\frac{1}{\Delta_{\min}})$ for any strategy, implies that
playing according to \ref{eq:lp2} is optimal up to at least the
phase at which $\Delta_r > \Delta_s$.

Next, consider the setting of $s$ s.t. $\Delta_r > \Delta_s$. The
following is feasible for \ref{eq:lp5_dual}
\begin{align*}
    y_i = \begin{cases}
    \frac{\Delta_r}{|\Gamma_s|} &\text{if } |\Gamma_s| \Delta_s \geq \Delta_r\\
    \Delta_s &\text{otherwise}.
    \end{cases}
\end{align*}
Thus the value of \ref{eq:lp2} is $\frac{\Delta_r}{\Delta_s^2}$ in
the first case and $\frac{|\Gamma_s|}{\Delta_s}$ in the second as we
can match these values in the primal by setting either $x_r =
\frac{1}{\Delta_s^2}$ or $x_i = \frac{1}{\Delta_s^2},i\in\Gamma_s$ in
the primal. To show that both of these values are dominated by $c^*$
consider the dual of \ref{eq:lp} below
\begin{equation}
\label{eq:lp_dual}
    \begin{aligned}
        \max_{y \in \RR^{K}}\sum_{i\in [K]\setminus I^*}&\frac{y_i}{\Delta_i^2}\\
        s.t.\sum_{j\in N_i\bigcap \Gamma_s} &y_j \leq \Delta_i \forall i\in[K].
    \end{aligned}\tag{LP6}
\end{equation}
First consider the setting in which the value of \ref{eq:lp2}
equals $\frac{\Delta_r}{\Delta_s^2}$. Set all $y_i \in \Gamma_s
\setminus I^*$ to $y_i=\frac{\Delta_r}{|\Gamma_s|}$ and all other
$y_i=0$. This is feasible for \ref{eq:lp_dual} and implies that
\begin{align*}
    c^* \geq \sum_{i\in \Gamma_s\setminus I^*}
    \frac{\Delta_r}{\Delta_i^2 |\Gamma_s|}
    \geq
    \frac{|\Gamma_s\setminus
      I^*|}{|\Gamma_s|}\frac{\Delta_r}{\Delta_s^2},
\end{align*}
where the second inequality follows because $\Delta_s \geq
\Delta_i$. This is sufficient to guarantee that $c^* +
\frac{|I^*|}{\Delta_{\min}}\geq \frac{\Delta_r}{\Delta_s^2}$.  Next
consider the setting in which the value of \ref{eq:lp2} equals
$\frac{|\Gamma_s|}{\Delta_s}$. Set all $y_i:i \in \Gamma_s \setminus
I^*$ to $y_i=\Delta_i$ and all other $y_i=0$. This is again feasible
for \ref{eq:lp_dual} because $\sum_{i \in \Gamma_s \setminus I^*}
\Delta_i \leq |\Gamma_s|\Delta_s \leq \Delta_r$ and further implies
that
\begin{align*}
     c^* \geq \sum_{i\in \Gamma_s\setminus I^*} \frac{1}{\Delta_i} \geq \frac{|\Gamma_s\setminus I^*|}{|\Gamma_s|}\frac{1}{\Delta_s}.
\end{align*}
Again this is sufficient to guarantee that $c^* +
\frac{|I^*|}{\Delta_{\min}}\geq \frac{|\Gamma_s|}{\Delta_s}$.
\end{proof}

\subsection{Proof of Lemma~\ref{lem:dstar_gen}}
\begin{proof}
Again we assume that $G$ consists only of a single connected component. First we argue that $\mathscr{C}(G)$ is a star-graph.
Consider three vertices $v_{1},v_{2},v_{3} \in \mathscr{C}(G)$ such that $v_1, v_3 \in N_{v_2}$. 
Assume that $v_1 \in N_{v_3}$. This implies that there exists a vertex $u \in \mathscr{C}(G)$ such that $u \in N_i$ but $u \not \in N_j$ for $i\neq j, i,j \in \{1,2,3\}$, otherwise $N_{v_1}=N_{v_2}=N_{v_3}$ and they collapse to a single vertex under $\mathscr{C}(G)$. Assume that $u \in N_1$ but $u \not \in N_2$. Then this implies there exists a path of length $3$ between $u$ and $v_2$, given by $(u,v_1,v_3,v_2)$. All other cases are symmetric and so this contradicts $v_2 \in N_{v_3}$. Further, it can not occur that there exists a neighbor $u$ of $v_2$ or $v_3$ s.t. $u\not\in N_{v_1}$. The above two arguments show that for $G$ every vertex must neighbor $v_1$ and no two vertices $v_2,v_3 \neq v_1$  can be neighbors making $\mathscr{C}(G)$ a star graph.

Next, we show that for any $\mu$ there exists a $\mu'$ defined on $\mathscr{C}(G)$ s.t. $c^*(G,\mu) = c^*(\mathscr{C}(G),\mu')$ and $d^*(G,\mu) = d^*(\mathscr{C}(G),\mu')$. For any equivalence class $[v] \in \mathscr{C}(G)$, define the expected reward of $[v]$ as $\mu'_{[v]} = \max_{u\in [v]} \mu_u$. For the remainder of the proof we represent the equivalence class by the action $v$ with maximum reward $\mu_v=\mu'_{[v]}$.
By construction, $\max_v \mu'_{[v]}=\max_u\mu_u$, hence the gaps are also identical $\Delta_{[v]}=\Delta_{v}$.
We first show that we can drop the constraints for any $u\in[v]\setminus\{v\}$ without changing the value of the LP. The LHS of all constraints for $u\in[v]$ is identical since it depends only on $N_u$.
Hence, we can remove all but the largest constraint, which is obtained for the smallest gap, i.e. the constraint for $v$.
Next we show that we can also remove $x_u$ for any $u\in[v]\setminus\{v\}$.
Assume $x_u > 0$ is a feasible solution of the LP, then we obtain another feasible solution $x'$ by $x'_u = 0, x'_v=x_u+x_v$, while leaving everything else unchanged. 
However, since $\Delta_{v}\leq \Delta_{u}$, the objective value of $x'$ is smaller or equal that of $x$.
Hence there exists an optimal solution where all $u\in[v]\setminus\{v\}$ are $0$ and these variables can be dropped from the LP. The resulting LP after dropping constraints and variables for $u\in[v]\setminus\{v\}$ is exactly given by $\mathscr{C}(G),\mu'$. This shows that $c^*(\mathscr{C}(G),\mu') = c^*(G,\mu')$.

The claim that $d^*(G,\mu) = d^*(\mathscr{C}(G),\mu')$ follows analogously.

\end{proof}

\end{document}